\documentclass{article}

\usepackage{microtype}
\usepackage{graphicx}
\usepackage{subfigure}
\usepackage{booktabs} % for professional tables

\usepackage{hyperref}

\usepackage[accepted]{icml2023}

\usepackage{amsmath}
\usepackage{amssymb}
\usepackage{mathtools}
\usepackage{amsthm}

\usepackage{graphicx}
\usepackage{wrapfig}
\usepackage{nicefrac}

\usepackage[utf8]{inputenc} % allow utf-8 input
\usepackage[T1]{fontenc}    % use 8-bit T1 fonts
\usepackage{hyperref}       % hyperlinks
\usepackage{url}            % simple URL typesetting
\usepackage{booktabs}       % professional-quality tables
\usepackage{amsfonts}       % blackboard math symbols
\usepackage{nicefrac}       % compact symbols for 1/2, etc.
\usepackage{microtype}      % microtypography
\usepackage{xcolor}         % colors
\usepackage{amssymb}        % intercal

\usepackage[capitalize,noabbrev]{cleveref}

\theoremstyle{plain}
\newtheorem{theorem}{Theorem}[section]

\theoremstyle{definition}

\theoremstyle{remark}
\newtheorem{remark}[theorem]{Remark}

\usepackage[textsize=tiny]{todonotes}

\icmltitlerunning{Nerflix And Severance}

\begin{document}

\twocolumn[
\icmltitle{Netflix and Forget: \\Efficient and Exact Machine Unlearning from Bi-linear Recommendations}


\icmlsetsymbol{equal}{*}

\begin{icmlauthorlist}
\icmlauthor{Mimee Xu}{sch}
\icmlauthor{Jiankai Sun}{comp}
\icmlauthor{Xin Yang}{comp}
\icmlauthor{Kevin Yao}{comp}
\icmlauthor{Chong Wang}{comp}
\end{icmlauthorlist}

\icmlaffiliation{comp}{Bytedance Inc., Bellevue, WA, USA}
\icmlaffiliation{sch}{Courant Institute of Mathematical Sciences, New York University, NY, USA}
%Work done during internship at ByteDance.
\icmlcorrespondingauthor{Mimee Xu}{mimee@nyu.edu}

\icmlkeywords{Machine Learning, Machine Unlearning, Applications, Recommendations}

\vskip 0.3in
]



\printAffiliationsAndNotice{}  % leave blank if no need to mention equal contribution

\begin{abstract}
%Suppose a person, who has streamed rom-coms exclusively with their significant other, suddenly breaks up.
%Consider an expecting mom, who has shopped for baby clothes, miscarries.
%Their streaming and shopping recommendations, however, do not necessarily update, serving as unhappy reminders of their loss.

People break up, miscarry, and lose loved ones.
Their online streaming and shopping recommendations, however, do not necessarily update, and may serve as unhappy reminders of their loss.
When users want to renege on their past actions, they expect the recommender platforms to erase selective data at the model level.
Ideally, given any specified user history, the recommender can unwind or "forget", as if the record was not part of training.
To that end, this paper focuses on simple but widely deployed bi-linear models for recommendations based on matrix completion.
Without incurring the cost of re-training, and without degrading the model unnecessarily, we develop Unlearn-ALS by making a few key modifications to the fine-tuning procedure under Alternating Least Squares optimisation, thus applicable to any bi-linear models regardless of the training procedure.
We show that Unlearn-ALS is consistent with retraining without \emph{any} model degradation and exhibits rapid convergence, making it suitable for a large class of existing recommenders.

%Motivated by real-world privacy, we seek empirical methods to implement
%the Right To Be Forgotten, specifically for recommendation systems built through collaborative filtering.
%The goal is to update downstream recommendations to reflect the removal of random training data without incurring the cost of re-training.

%Our method, Unlearn-ALS, is inspired by bi-linear solutions  and can be applied directly to trained bi-linear recommendations regardless of the training procedure. It is \emph{exact}: consistent with retraining in theory without any approximation or model degradation, while showing fast convergence in practice.

\end{abstract}
%We assume that users may retract consent to having some of their data included in the training of machine learning models, and that technology platforms will be compelled to remove their influence.

\section{Introduction}
Break-ups, pregnancy losses, and bereavements are particularly painful in the age of ubiquitous machine learning systems.
Per General Data Protection Regulation (GDPR), an individual may request their personal data erased under the right to erasure, or "Right To Be Forgotten"~\citep{EUdataregulations2018, voigt2017eu}.

%Suppose a user watches a Korean drama with their significant other but breaks %up mid-season. Signing into Netflix, they are notified of new episodes of the %show that they started with their ex-partner.
%While browsing for something new, they are suggested TV series that feature
%the same actors as the unfinished show. If their taste for "K-Drama" was only %through the relationship, encountering the sub-genre may cause distress.
%To move on, the user may reasonably demand some past records expunged from %Netflix's recommendation engines~\citep{thompson_2021}.

Suppose a Netflix user watches a Korean drama with their significant other but breaks up mid-season. They are subsequently bombarded with new episode alerts and recommended shows with the same actors and art styles, potentially causing distress. To move on, the user may wish that Netflix recommenders expunge some of their watch history. %past records from recommendation engines.%~\citep{thompson_2021}.

The platform could accommodate deletion, not only in user
history but also in subsequent recommendations.
Ideally, this deletion is both swift and seamless.
An incomplete "under-deletion" likely persists the underlying concepts learned from the deleted records, preventing the user from cultivating a new path forward due to "echo-chamber" style feedback~\citep{chaney2018algorithmic, jiang2019degenerate, mansoury2020feedback}.
Yet, an "over-deletion" may needlessly degrade model utility; for instance, a callous reset can cause degeneracy where trendy items are conspicuously missing.

% This deletion
% must be (1) swift, as retraining is expensive, and (2) well-calibrated, so the user can move on seamlessly. Due to "echo-chamber" feedback loops~\citep{chaney2018algorithmic, jiang2019degenerate, mansoury2020feedback}, an incomplete "under-deletion" likely persists the underlying concepts learned from the deleted records, preventing the user from cultivating a new path forward. Yet, a callous reset may needlessly hurt the utility of the system. Worse, an "over-deletion" may cause certain recommendations to be conspicuously missing, introducing new privacy risks by revealing to other observers which memories were removed.

Fortunately, many deployed recommendation systems are uncomplicated: assuming low-rank user and item features, they solve a matrix completion problem. For building industrial recommenders, a widely-used optimization is Alternating Least Squares (ALS) which is fast-convergent~\citep{hu2008collaborative, koren2009matrix, takacs2011applications, he2016fast}. We focus on data erasure from these practical systems.

Despite the simplicity of bi-linear recommendation models, few works explored performing unlearning from them post-training. Thus, the first step of our exploration examines when linear models of few parameters indeed fail to memorize samples, or are otherwise "robust" against a small number of random deletions, summarized in Section~\ref{sec:inherent}.

For arbitrary deletion, we develop Unlearn-ALS, which modifies the intermediate confidence matrix used in ALS to achieve fast forgetting. Mathematically, Unlearn-ALS is equivalent to minimizing the loss of the model on the remaining data by retraining with ALS, making it a method for \emph{exact} deletion. Section~\ref{sec:method} presents its analyses and results.

We further ask, is our work done? In a similar setting, industrial recommendations and systems trained with differential privacy do leak training data with specific users ~\citep{calandrino2011you, rahman2018membership}, underscoring the importance of empirical evaluation. While our theory works well in the random deletion setting, in practice, however, there may still be privacy risks with respect to the deleted data, in the model after performing deletion. In completion, we develop a membership inference variant to evaluate the privacy risks of our unlearning procedure.

\paragraph{Our contributions} (1) We clarify that practical bi-linear recommendation models can have privacy risks from memorizing training data. (2) We propose Untrain-ALS, a crafty and fast heuristic that unlearns a bi-linear model, and makes no compromise to recommendation accuracy. (3) We devise an empirical test using de-noised membership inference, which is more sensitive to bi-linear models' memorization.
\section{Related Works}
\paragraph{Machine unlearning} is an emerging field motivated by performance and computation trade-offs to implement the Right to Be Forgetten on machine learning models~\citep{grau2006eternal}.
When a user seeks to retract data used in training, the derived model ought to
update with respect to the change. Unlearning thus trades off computation, accuracy, and privacy, and is often compared with retraining~\citep{pmlr-v132-neel21a, NEURIPS2019_cb79f8fa, bourtoule2021machine, golatkar2020eternal}.
\paragraph{Unlearning recommendation systems} is concurrently explored by
~\citet{li2022making} and~\citet{chen2022recommendation}, which target unlearning for industrial scale recommendations built through collaborative filtering. Sharding and user clustering are key to their methods, which we do not consider. Instead, our work complements the line of work through a much simpler unlearning algorithm that applies to all bi-linear models with minimal architectural change.
\paragraph{Differentially-private recommendations}
~\citet{mcsherry2009differentially, liu2015fast} may be naturally compliant towards the Right to Be Forgotten by reducing the risk related to the model output revealing information about the inclusion of certain data. However these methods would need to anticipate to a certain extent the likelihood of deletion, and build that into training.

\paragraph{Evaluations against privacy risks} if no privacy risk is shown, it would mean that no computation needs to be expended on unlearning. Membership Inference is a popular method that measures training data memorization by a model. Typical membership inference uses a collection of samples that are not in the training data, feed them to the model, and take the outputs as the baseline negative training set. The positive training set is the data that the model has seen in the training set. Other membership inference methods have been developed, usually requiring access to the model or the training procedure more metrics~\citep{chen2021machine}. The central idea is to make the empirical attack model more powerful.

Recently, ~\citet{46702} took a different approach. They developed a very effective empirical evaluation would be applicable to any model after it has been trained. For large scale language models, feature injection can test if a data point had been deleted~\citep{izzo2021approximate}. This negative dataset is manufactured "poison" to the training procedure. The intuition is that if the model is prone to memorization, it would be able to reproduce the exact random string that was injected in the training set. The membership inference variant thus focuses on engineering a better dataset, thus making it more effective at uncovering memorization. While powerful, it requires internal access to model training.

\paragraph{Differential Privacy}
Similar to a well-behaving matrix completion solution's inherent privacy (Section~\ref{sec:inherent}), some models may be less prone to memorizing individual data points. As a result, they are less at risk for membership attacks after deletion requests.

By definition, pure differentially private models are robust to deletion, as each individual data point's membership should not be inferrable~\citep{dwork2009differential}. Yet, not all models trained with differential privacy are robust. In practice, assumptions on the independences between data points do not hold and the number of deletion requests may not be known ahead of training; additionally, businesses often opt for approximations, since pure differential privacy poses degradation on model utility. As a result, ~\citet{rahman2018membership} finds that models trained to be differentially private are yet vulnerable.
\section{Preliminaries}
\label{sec:prelim}
%This section sets up the problem statement: to perform Machine Unlearning~\citep{bourtoule2021machine} from the learned recommendations when users send deletion requests.
\begin{table}[t]
\caption{Notations for Untrain-ALS comparisons.}
\label{tab:notation}
\vskip 0.15in
\begin{center}
\begin{small}
\begin{sc}
\begin{tabular}{lcccr}
\toprule
Model & Datasets & Baseline For\\
\midrule
 $\mathcal{M}_\mathrm{undeleted}$& $\mathcal{D}_\mathrm{obs}$    & Performance \\
$\mathcal{M}_\mathrm{retrain}$& $\mathcal{D}_\mathrm{remain}$    & Privacy Loss \\
 $\mathcal{M}_\mathrm{untrain}$& $\mathcal{D}_\mathrm{obs}$, $\mathcal{D}_\mathrm{removal}$& (Our Method) \\
\bottomrule
\end{tabular}
\end{sc}
\end{small}
\end{center}
\vskip -0.1in
\end{table}
\begin{figure}
\begin{minipage}[t]{0.47\textwidth}
\centering
\includegraphics[width=0.98\linewidth]{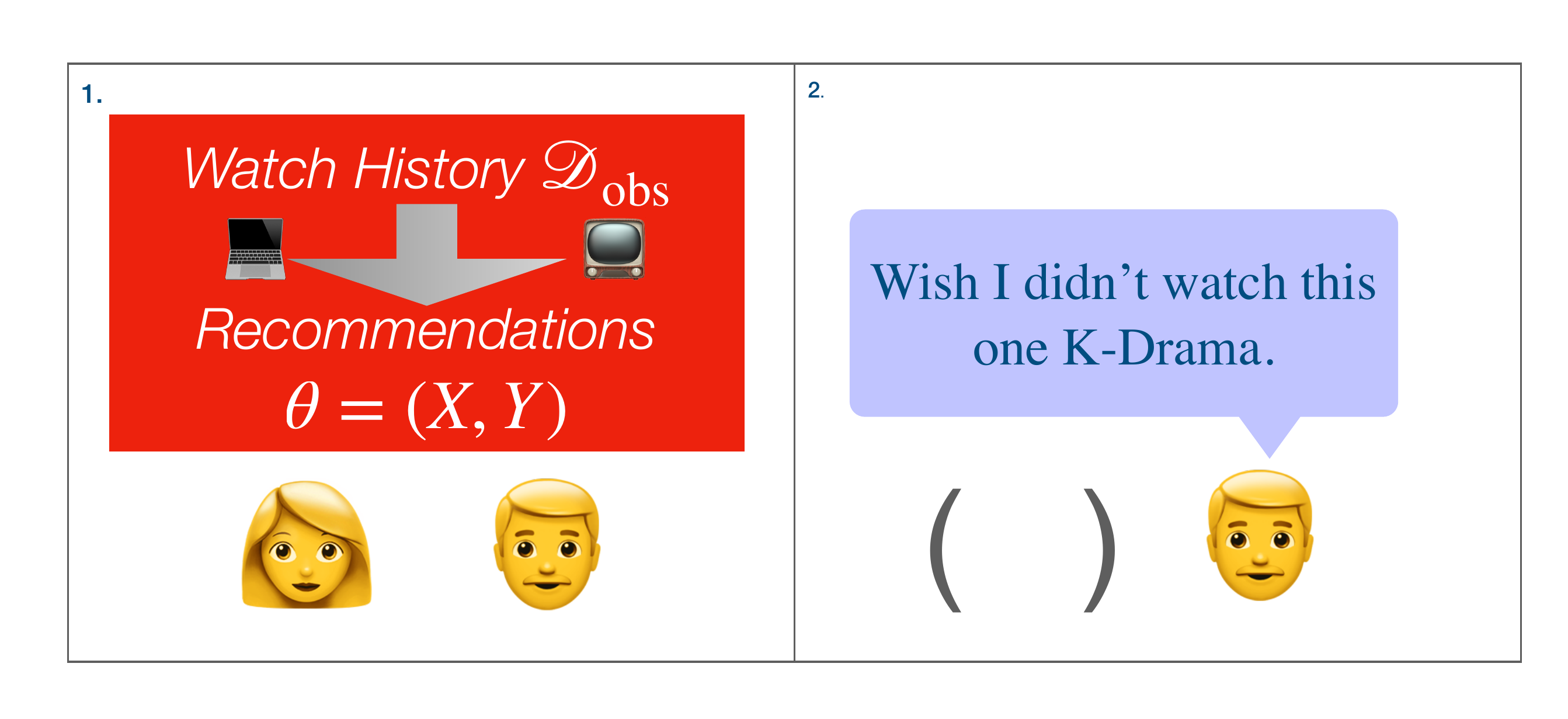}
    \caption{A user suffers a sudden breakup, and requests the recommendation owner Netflix to erase selective watch histories.}
\end{minipage}

\hfill
\begin{minipage}[t]{0.47\textwidth}

\end{minipage}
\end{figure}

\textbf{Matrix Completion}. We assume a base collaborative filtering model based on matrix factorization, learned through a user-item ratings matrix $P$ as in MovieLens~\citep{bennett2007netflix}. The downstream recommendation for each user is given based on the ranking of items~\citep{koren2009matrix}.

Assume matrix $M$ where $m_{ij}:=M[i][j]$ denotes the \emph{ground truth} preference of user $i$ with respect to item $j$. The entries of $P$ are assumed sampled from $M$; if the interaction is not observed, $p_{ij}=0$. In matrix factorization, $M$ can be recovered through a low rank multiplication,
\begin{equation}
\label{eq:matrix_completion}
M = X Y^\mathsf{T} .
\end{equation}
where $X$ depicts user features over all users, and $Y$ is the underlying item factors e.g. movie features.
Given only $P$, we aim to recover $X, Y$.

\begin{figure*}[t]
\begin{minipage}[t]{0.47\textwidth}
\begin{algorithm}[H]
\caption{AlternatingLeastSquares}
\label{alg:als}
    \begin{algorithmic}
        \REQUIRE $P$, $\alpha$, $\lambda$, initialize $X, Y$ randomly.
        \STATE $c_{ui} \gets 1 + \alpha p_{ui}$ \COMMENT{~\citep{hu2008collaborative}}
        \WHILE{model does not converge}

        \FORALL{$u$}
        %\State $x_u \gets (Y^\intercal C^u Y + \lambda I)^{-1}C^u p(u)$
        \STATE $x_u \gets (Y^\intercal C^u Y + \lambda I)^{-1}C^u P^u$
        \ENDFOR

        \FORALL{$i$}
        %\State $y_i \gets (X^\intercal C^i X + \lambda I)^{-1}C^i p(i)$
        \STATE $y_i \gets (X^\intercal C^i X + \lambda I)^{-1}C^i P^i$\;
        \ENDFOR
        \ENDWHILE
    %\STATE \RETURN $X, Y$ as $\hat{X}, \hat{Y}$
    %\State $(\hat{X}, \hat{Y}) \gets (X, Y)$
    \end{algorithmic}
\end{algorithm}

\end{minipage}
\hfill
\begin{minipage}[t]{0.47\textwidth}
\begin{algorithm}[H]
    \caption{Untrain-ALS}
    \label{alg:untrain_als}
        \begin{algorithmic}
            \REQUIRE $P, \alpha, \lambda, \hat{X}, \hat{Y}, C_0, \mathcal{D}_\mathrm{removal}$
    %        \Ensure $\alpha = 40$
            \STATE $X, Y \gets \hat{X}, \hat{Y}$ \COMMENT{from Algorithm~\ref{alg:als}}
            \FORALL{$(u,i)\in\mathcal{D}_\mathrm{removal}$}
            \STATE $p_{ui}\gets 0, c_{ui}\gets 0$ \COMMENT{delete and block}
            \ENDFOR
            \WHILE{model does not converge}

            \FORALL{$u$}
            % \State $x_u \gets (Y^\intercal C^u Y + \lambda I)^{-1}C^u p(u)$
            \STATE $x_u \gets (Y^\intercal C^u Y + \lambda I)^{-1}C^u P^u$
            \ENDFOR

            \FORALL{$i$}
            %\State $y_i \gets (X^\intercal C^i X + \lambda I)^{-1}C^i p(i)$
            \STATE $y_i \gets (X^\intercal C^i X + \lambda I)^{-1}C^i P^i$
            \ENDFOR
            \ENDWHILE
            %\State \Return $X, Y$ as $\hat{X}, \hat{Y}$
        \end{algorithmic}
\end{algorithm}
\end{minipage}
\end{figure*}
\textbf{Alternating Least Squares (ALS)}. Unless otherwise mentioned, we simulate training (and re-training) with AlternatingLeastSquares (ALS), a widely deployed heuristic by~\citet{hu2008collaborative, takacs2011applications} outlined in Algorithm~\ref{alg:als}. ALS is exceedingly simple and parallelizable; despite having little theoretic guarantee it converges fast empirically for recommendation data~\citep{koren2009matrix, jain2013low, uschmajew2012local}.

For given ratings matrix $P$ and desirable rank $k$, we learn the model parameters $\hat{\theta} = \{\hat{X}, \hat{Y}\}$. The regularized matrix completion with parameter $\lambda$ also associates each entry with a confidence score $c_{ui}$. Using $\mathcal{D}_\mathrm{obs} = \{(u, i)\}$ to denote the coordinates of $M$ that contain explicit observations\footnote{$m_{ui}\neq 0$ $\forall(u, i)\in \mathcal{D}_\mathrm{obs}$}, the loss function $L_\mathrm{ALS}(X, Y)$ is written as
\begin{equation}
\label{eq:als}
\sum_{(u, i) \in \mathcal{D}_\mathrm{obs}} c_{ui} (p_{ui} - x^\intercal_u y_i)^2 + \lambda(\sum_u||x_u||^2 + \sum_i ||y_i||^2).
\end{equation}

Algorithm~\ref{alg:als} makes a non-convex optimization convex at each of the alternating minimizations. To tackle implicit feedback, a confidence matrix $C$ is constructed as a soft copy of the ratings, where $c_{ui}:= 1 + \alpha p_{ui}$ for $\alpha \in \mathbf{R^+}$: if the ratings were high, the confidence is high, and if the ratings are missing, the confidence is low. $C$ is then used throughout the iterations instead of $P$.

Though we treat ALS as the baseline ground truth for training (and re-training), our unlearning algorithm, Untrain-ALS, applies to any bi-linear model. See Appendix~\ref{app:als} for experiment parameters.

\textbf{Additional assumptions}.
The removal set, $\mathcal{D}_\mathrm{removal}$, is uniformly sampled from $\mathcal{D}_\mathrm{removal}$ without replacement, and it cannot be known prior to training. Our theoretical analysis replies on uniform sampling. Further, the coordinates in $\mathcal{D}_\mathrm{obs}$ are assumed to be i.i.d., to ensure that models trained without access to the deleted data are statistically independent from the removal set. Lastly, $|\mathcal{D}_\mathrm{obs}|\gg|\mathcal{D}_\mathrm{removal}|$ to simulate occasional deletion requests.

\textbf{Re-training as privacy baseline}. As our goal is to neither over- nor under-delete, the ideal removal of $P[m][n]$ is to train another model with new preference matrix $P'$ where $P'[m][n] = 0; P'[i][j] = P[i][j]\, \mathrm{otherwise}$. The retrained model will thus treat the removed samples as simply missing data, as ~\citet{hu2008collaborative}'s \emph{implicit} feedback, ensuring privacy requirements. Additionally, we are only concerned with cases where $P_{mn} \neq 0$ so that the deletion is meaningful.

\textbf{Empirical evaluations.} In our setup, after unlearning procedure, the removed data should be indistinguishable from unobserved data. In Membership Inference (MI), the trained model's outputs can be exploited to judge whether a data sample was part of the training data. Typically, an MI classifier $\sigma(\mathcal{M}):(x)\to \{0, 1\}$ is a binary logistic regressor. Our MI training set is constructed with positive data of actual training samples' outputs, and negative data of removed training samples' outputs.

Nonetheless, a robust unlearning does not require an associated low MI accuracy. Instead, we are concerned with \emph{increased} confidence in membership attack caused by the unlearning procedure.

\textbf{Vulnerability}. Fixing the training procedure, the re-trained model and the trained model can be seen as a function of their observed ratings matrix. Let $\mathrm{MI}(\cdot):(\theta, \mathcal{D}_\mathrm{removal}, \mathcal{D}_\mathrm{remain}) \to [0,1]$, which refers to the membership inference accuracy on a particular model given the removal set and the remaining set. Because all the evaluations fix the datasets between retraining and untraining, we simply write $\mathrm{MI}(\mathrm{untrain})$ to refer to membership inference accuracy with untraining.

Typically, $\mathrm{MI}$ is directly used as a vulnerability measure. As we compare against re-training from scratch, the \emph{additional} vulnerability caused by the choosing untraining over retraining is written as $\mathrm{MI}(\mathrm{untrain}) - \mathrm{MI}(\mathrm{retrain})$. In Section~\ref{sec:mi}, we propose instead to use $\mathrm{MI}(\mathrm{unlearn}) - \mathrm{MI}(\mathrm{train}) -\mathrm{MI}(\mathrm{undeleted})$ under fixed data splits, to denoise the effect of the base undeleted model.

\section{Intuitions on default privacy against random data removal in matrix completion}
\label{sec:inherent}
A model is defaultly private when it does not need to change under the removal of some training data. When user-requested deletions are drawn uniformly from user data, two factors indicate "inherent" robustness of a matrix completion model: 1. having high test accuracies (or trained with calibration) and 2. the low rank nature of the data.
%the observed user-item entries $\mathcal{D}_\mathrm{obs}$

We comb through these arguments, extended n Appendix~\ref{app:inherent}, and present why empirical methods are still needed.
\begin{figure}[t]
  \centering
  \includegraphics[width=0.47\textwidth]{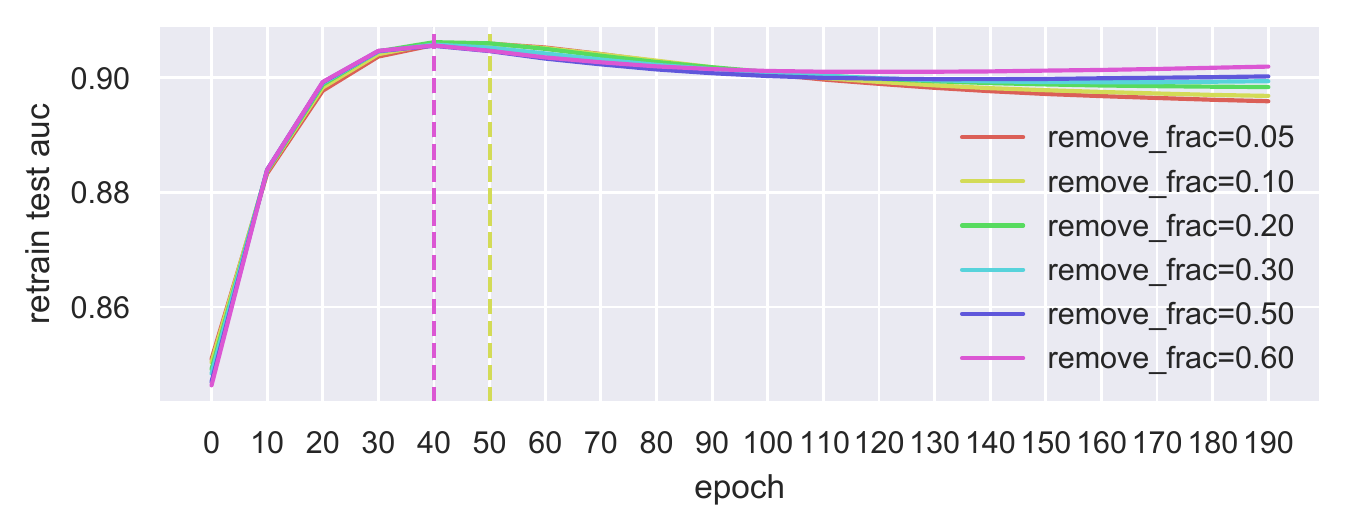}
  \caption{\textbf{Baseline: re-training dynamics.} When training from scratch, model AUC on same test set (1\% of MovieLens-1M) across 6 different fractions of removal. Dotted: the optimal number of training iterations. Results checkpointed at every 10 iterations.}
  \label{fig:auc_1m_retrain}
\end{figure}

%This section sketches the intuitions, and is extended in Appendix~\ref{app:inherent}.

\textbf{Arguments for inherent privacy}. First, implicit feedback datasets mark missing data as $0$ in $P$. The model training treats unobserved entries as zero, including $\mathcal{D}_\mathrm{test}$.
%that is held out for validation.

Second, identical sampling for user requests $\mathcal{D}_\mathrm{removal}$ and $\mathcal{D}_\mathrm{test}$. For $x_r\in \mathcal{D}_\mathrm{removal}, x_t\in \mathcal{D}_\mathrm{test}, x_r, x_t \stackrel{i.i.d.}{\sim} \mathrm{U}(\mathcal{D}_\mathrm{obs})$. Naturally $\mathcal{M}_\mathrm{undeleted}$ cannot distinguish between a sample drawn from one set or another based on query output as in ~\citet{carlini2022membership}'s "membership inference security game", or $\mathbb{E}[L_{x\sim \mathcal{D}_\mathrm{removal}}(x)] = \mathbb{E}[L_{x\sim \mathcal{D}_\mathrm{test}}(x)]$.

Moreover, empirical recommenders can be highly accurate without memorization. At the optimal AUC in Figure~\ref{fig:auc_1m_retrain}, the base model predicts well even with large removal fractions. Further, in linear models, per~\citet{kearns1995bound} and ~\citet{blum1999beating}, appropriate calibration in training results in in-domain generalization, expecting low prediction losses on missing data for both $\mathcal{M}_\mathrm{retrained}$ and $\mathcal{M}_\mathrm{undeleted}$.

Lastly, the impact of using a small number of parameters in matrix completion is extended in Appendix~\ref{app:inherent}. Rehashing the key claim in ~\citet{recht2011simpler} we show that the exact solutions to matrix completion is inherently robust to randomly sampled deletions under data coherence assumptions.

\textbf{Real-world ALS training breaks theoretic assumptions}. Model training typically employs regularization (Equation~\ref{eq:als}), and early-stopped at the best fit (Algorithm~\ref{alg:als}), not to completion. Matrix coherence of real world data, as ~\citet{recht2011simpler} requires, is not testable in practice. Lastly, the decompositions learned using ALS can be non-unique (nor equivalent up to a rotation)~\citep{jain2013low}, so the removal samples may be especially vulnerable with respect to the currently deployed model, thus requiring manual deletion.

Nevertheless, average membership attack accuracies may be especially low against a matrix completion model. An attacker aims to discriminate between the predictions influenced by $\mathcal{D}_\mathrm{removal}$ and $\mathcal{D}_\mathrm{remain}$ after seeing some samples from each~\citep{carlini2022membership}.
Varying data splits, a well-calibrated $\mathcal{M}$ has similar expected losses across. Optimizing for area-under-curve (AUC) is used in 1. thresholding membership inference model on the removal data and 2. selecting the optimal model, we have $\mathbb{P}_{(u,i)\sim \mathcal{D}_\mathrm{removal}}(p_{ui} = 1)\approx \mathbb{P}_{(u,i)\sim \mathcal{D}_\mathrm{obs}}(p_{ui} = 1) = \mathrm{AUC}$. For both $\mathcal{M}_\mathrm{retrain}$ and $\mathcal{M}_\mathrm{undeleted}$, lower validation loss hinders attacker accuracy, making the \emph{difference} of the attacks i.e. the privacy improvements from re-training numerically small, which we discuss further in Section~\ref{sec:modify-mi} and discard the averaging across data splits for more effective membership inference.

% we have $P_\mathrm{retrained}(p_{ui} = 1| (u,i)\sim \mathcal{D}_\mathrm{rm})\approx P_\mathrm{retrained}(p_{ui} = 1| (u,i)\sim \mathcal{D}_\mathrm{obs}) = \mathrm{AUC}_\mathrm{retrain}$. For each model, the approximation is directly relatable to validation loss. We discuss the implications in Section~\ref{sec:modify-mi}, where we discard the averaging across data splits for more effective membership inference.
\section{Exact Deletion with Untrain-ALS: Untraining Alternating Least Squares}
\label{sec:method}
Our unlearning strategy, Untrain-ALS outlined in Algorithm~\ref{alg:untrain_als}, makes slight modifications to the fast ALS heuristic used in training implicit feedback recommendations.
\begin{enumerate}
\item \textbf{Pre-train}. Use the resulting user and item features $X_0$, $Y_0$ in Algorithm~\ref{alg:als} to initialize ALS.

\item \textbf{Deleting preferences}. Set $p_{ui} = 0$ for deleted item-user interaction $i, u$, common practice for fine-tuning.

\item \textbf{Blocking confidence on removed data}. Set $c_{ui}\leftarrow 0$ for deleted item-user interaction $i, u$ at all subsequent iterations.  Crucially this prevents further influence of the deleted data, thus allowing the model to refit to the remaining data fast. Optionally, use adjusted inverse.
% [both: un-training, maybe only doing 2 for ablation study].
\end{enumerate}
%The hope is that the model "rewinds" the influence of the removal data. Unlike gradient methods, ALS can make large steps.

%\textbf{Unlearn-ALS Intuition}. Steps 1-2 emulate fine-tuning; unlike gradient methods, ALS can make large steps. Step 3 prevents further influence of the deleted data, thus allowing the model to refit to the remaining data fast. The hope is that the model "rewinds" the influence of the removal data.

\subsection{Untrain loss = retrain loss, functionally}
\label{sec:theory}
Recall that the holy grail of unlearning is to approximate retraining. Under these modifications to $p_{ui}$ and $c_{ui}$, Untrain-ALS objective is functionally equivalent to re-training.
\begin{theorem}
$L_\mathrm{UntrainALS} = L_\mathrm{retrain}$ on $\mathcal{D}_\mathrm{obs}, \mathcal{D}_\mathrm{removal}$.
\end{theorem}

\begin{proof} [Proof Sketch]
\begin{multline*}
L_\mathrm{UntrainALS}(\mathcal{D}_\mathrm{obs}, \mathcal{D}_\mathrm{rm})
= \\ \sum_{(u, i) \in \mathcal{D}_\mathrm{obs}} f_c(c_{ui}) (p_{ui} - x^\intercal_u y_i)^2
+ \lambda(\sum_u||x_u||^2 + \sum_i ||y_i||^2)
\end{multline*}

where $f_c(\cdot)$ transforms the confidence score. Using Kronecker delta $\delta$ for set membership, our algorithm has
\begin{align*}
f_c(c_{ui})& = \delta_{(u,i)\in( \mathcal{D}_\mathrm{obs} \backslash \mathcal{D}_\mathrm{rm})}c_{ui} \\
& =  (1-\delta_{(u,i)\in \mathcal{D}_\mathrm{rm}})c_{ui} = c_{ui}-\delta_{(u,i)\in \mathcal{D}_\mathrm{rm}}c_{ui}.
\end{align*}

Expanding untraining loss along $\mathcal{D}_\mathrm{remain}$ and $\mathcal{D}_\mathrm{removal}$,
\begin{equation*}
\begin{split}
L_\mathrm{UntrainALS} &= \lambda(\sum_{u\in \mathcal{D}_\mathrm{obs}}||x_u||^2 + \sum_{i \in \mathcal{D}_\mathrm{obs}}||y_i||^2)
\\ &+
\sum_{(u, i) \in \mathcal{D}_\mathrm{remain}} c_{ui} (p_{ui} - x^\intercal_u y_i)^2
\\&+
\sum_{u, i \in \mathcal{D}_\mathrm{removal}} (0) (p_{ui} - x^\intercal_u y_i)^2
\end{split}
\end{equation*}
Because random requests $|\mathcal{D}_\mathrm{removal}| \ll|\mathcal{D}_\mathrm{obs}|$, the set of contributing $u, i$ is not expected to change, therefore

\begin{equation*}
\begin{split}
\label{eq:retrain_loss}
L_\mathrm{UntrainALS} &= \lambda(\sum_{u\in \mathcal{D}_\mathrm{remain}}||x_u||^2 + \sum_{i \in \mathcal{D}_\mathrm{remain}}||y_i||^2)
\\ &+
\sum_{(u, i) \in \mathcal{D}_\mathrm{remain}} c_{ui} (p_{ui} - x^\intercal_u y_i)^2
\\& = L_\mathrm{ALS}(\mathcal{D}_\mathrm{remain}) \quad \text{(Appendix~\ref{app:proof} in full)}.
\end{split}
\end{equation*}
\end{proof}
%TODO:
% put in algorithms YES
% put in teaser chart NO (res/8mo ago notebook, res_auc/1m/10mo ago?)
% math for having the same loss (app!)
% train-untrain comparisons charts!.
% fickleness chart (teaser?)
% model fit discussion. - make it into main text
% put in math on deletion may be ok
% put in runtime chart

\begin{remark}
It may appear that with such strong results, our work is over. Yet again, two real-world issues prevent us from claiming any untrained model is the same as any retrained model:
 1. empirically, the models are trained with early stopping: the number of epochs to train is determined by minimal loss; and 2. matrix factorization solutions via ALS are not unique. For empirical privacy, some of the potential solutions may be \emph{more private} than others. Therefore, it is crucial to complement with empirical privacy measures.
\end{remark}
\subsection{Untrain Runtime $\leq$ Training Runtime, Per Pass} Clearly, every pass of Untrain-ALS has the same runtime as a pass of ALS (Algorithms~\ref{alg:als} and~\ref{alg:untrain_als}). UntrainALS benefits from convergence analyses of ALS itself~\citet{uschmajew2012local}. Because the loss of the pre-trained model is minimal, using Untrain-ALS would be much faster than doing ALS from scratch. Section~\ref{sec:exp_runtime} verifies empirically that UntrainALS takes fewer passes than re-training.

\textbf{Speedups}. Every default pass of ALS requires inverting a large matrix. Though fast implementations use conjugate gradient (CG) to approximate inverses~\citep{takacs2011applications}, we note a faster alternative for exactly computing the matrix inverse in Untrain-ALS, where the original inverse is already available.
% In realistic cases of the Right to Be Forgotten, the platform receives deletion requests that are much fewer than remaining data.
Adjusting for $c_{ui}\leftarrow 0$ is equivalent to changing a single entry in the diagonal matrix $C^u$. This subtraction of a one-entry matrix is the perturbation of concern. The resulting confidence matrix under un-training, $\widetilde{C^u}$, is very close to the original confidence matrix, where
 \begin{equation}
 \widetilde{C^u} := C^u -(\mathrm{diag}[0,\cdots,c_{ui},\cdots,0]).
 \end{equation}
Consider a special case of Woodbury's inverse~\citep{woodbury1950inverting} where only one element is subtracted, by ~\citet{sherman1950adjustment}'s subtraction case, for matrix $A$, there is $(A-uv^\intercal)^{-1} = A^{-1} + A^{-1}u(1-v^\intercal A^{-1} u)^{-1}v^\intercal A^{-1}$. Let $A:=Y^\intercal C^u Y + \lambda I$. The adjusted inverse becomes
 \begin{multline}
 (\widetilde{A})^{-1}
 = (Y^\intercal C^u Y + \lambda I)^{-1} \\
 +\frac{c_{ui}}{1 -q} y_i( Y^\intercal C^u Y + \lambda I)^{-1} y^\intercal_i(Y^\intercal C^u Y + \lambda I)^{-1}.
 \end{multline}

%\[(\widetilde{A})^{-1} = (Y^\intercal C^u Y + \lambda I)^{-1} \\+\frac{c_{ui}}{1 -q} y_i( Y^\intercal C^u Y + \lambda I)^{-1} y^\intercal_i(Y^\intercal C^u Y + \lambda I)^{-1}.\]

\textbf{Overall Per-Pass Runtime}. Without inverse adjustment, $O(|\mathcal{D}_\mathrm{obs}|k^2+n^3k)$ is both the ALS and Untrain-ALS runtimes. With inverse adjustment, every user feature is computed to complete one step of ALS: $X_u \leftarrow (Y^\intercal C^u Y + \lambda I)^{-1}C^u p(u)$. In ALS, the inverse of $\mathcal{A}$ is computed in $O(k^3)$, and using CG speeds it up to $O(k^2p)$ where $p$ is the number of CG iterations. Assuming $\mathcal{A}^{-1}$ has been computed in the pretraining step, we can see the adjustment as a perturbation on $\mathcal{A}$, which we project to its inverse. This allows for a runtime of $O(k^2)$ per user or item per iteration, or every Untrain-ALS pass $O(|\mathcal{D}_\mathrm{obs}|k^2)$ for single deletion.
% ALS and Untrain-ALS runtimes are both $O(|\mathcal{D}_\mathrm{obs}|k^2+n^3k)$. With the inverse adjustment, recall in ALS, the inverse of $A$ is computed in $O(k^3)$, and using CG speeds it up to $O(k^2p)$ where $p$ is the number of CG iterations. Assuming $A^{-1}$ has been computed in the pretraining step, the adjustment is a perturbation on $A$, which we project to its inverse. This allows for a runtime of $O(k^2)$ per user or item per iteration, making every Untrain-ALS pass $O(|\mathcal{D}_\mathrm{obs}|k^2)$.

\section{Numerical Results}
\label{sec:exp}
Extensive numerical simulations verify the conclusions in our method, and we empirically demonstrate the efficiency of Unlearn-ALS using MovieLens data~\citep{bennett2007netflix}. Appendix ~\ref{app:als} states experimental parameters.

\label{sec:exp_untrain}
\subsection{Experimental Goals and Setup}
%ADD MODEL NAMES HERE?

  \begin{figure*}[t]
      \begin{minipage}[t]{0.31\textwidth}
        \centering
        \includegraphics[width=0.98\textwidth]{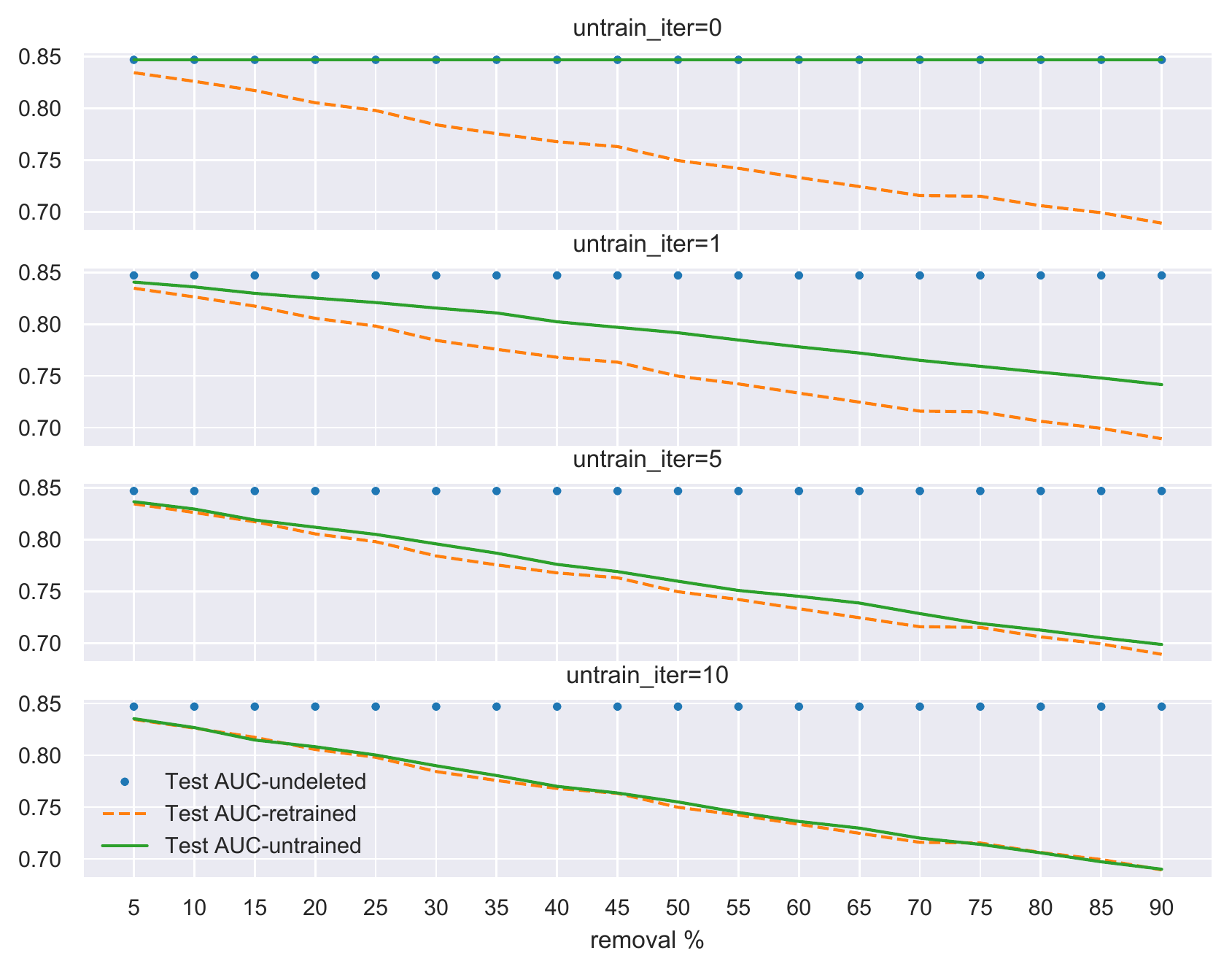}
        \caption{\textbf{Unlearning with Untrain-ALS.} Across different iterations of unlearning, and 20 different fractions of removal, each final model's area-under-curve on test set on MovieLens-100K. Retraining does 25 ALS passes.}
        \label{fig:auc_fit}
      \end{minipage}
      \hfill
      \begin{minipage}[t]{0.32\textwidth}
      \centering
      \includegraphics[width=0.98\textwidth]{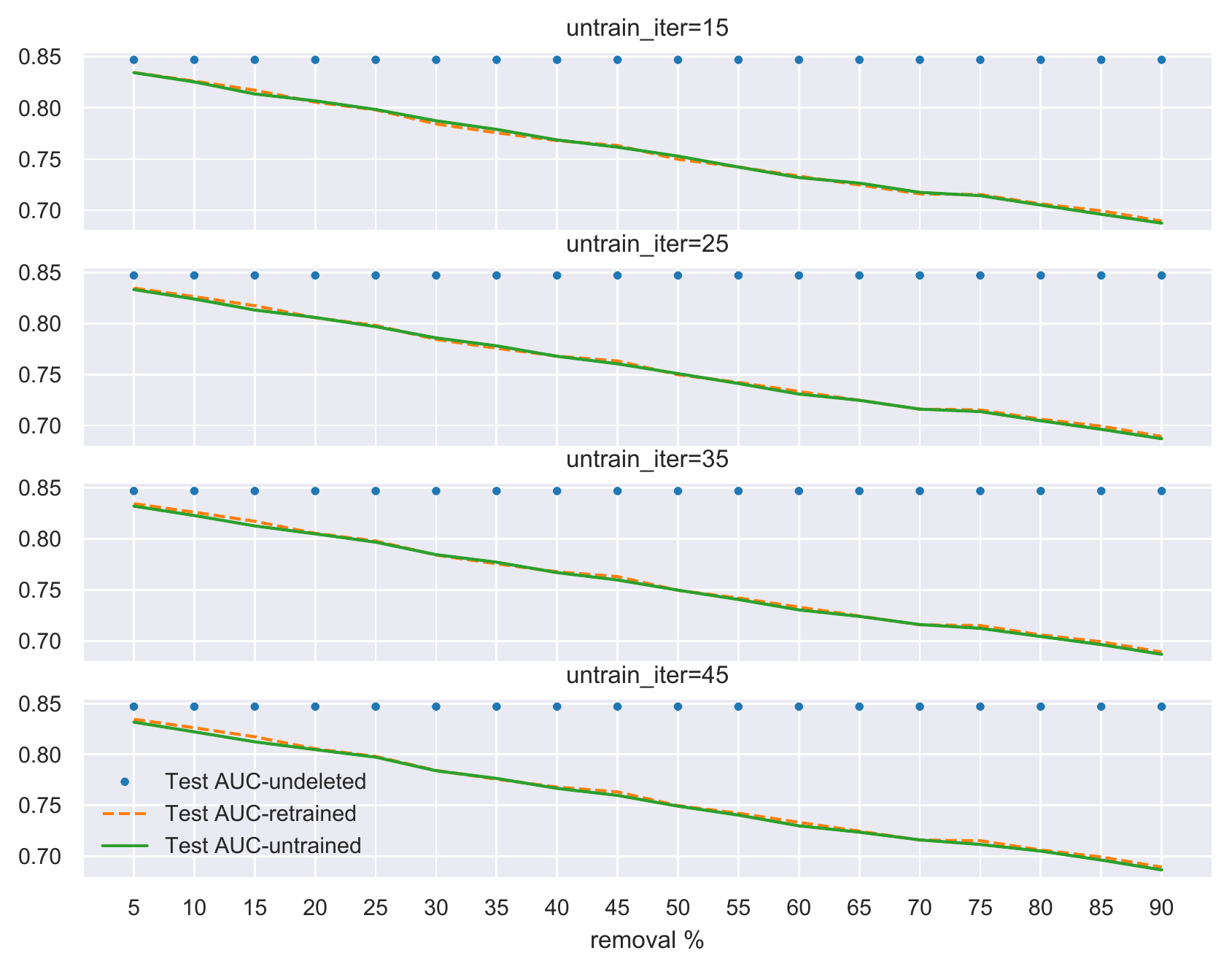}
      \caption{\textbf{Overfitting Untrain-ALS.} Across different iterations of unlearning, and 20 different fractions of removal, each final model's area-under-curve on test set on MovieLens-100K. Retraining consists of 25 ALS passes.}
      \label{fig:auc_overfit}
      \end{minipage}
      \hfill
      \begin{minipage}[t]{0.33\textwidth}
      \centering
      \includegraphics[width=0.95\textwidth]{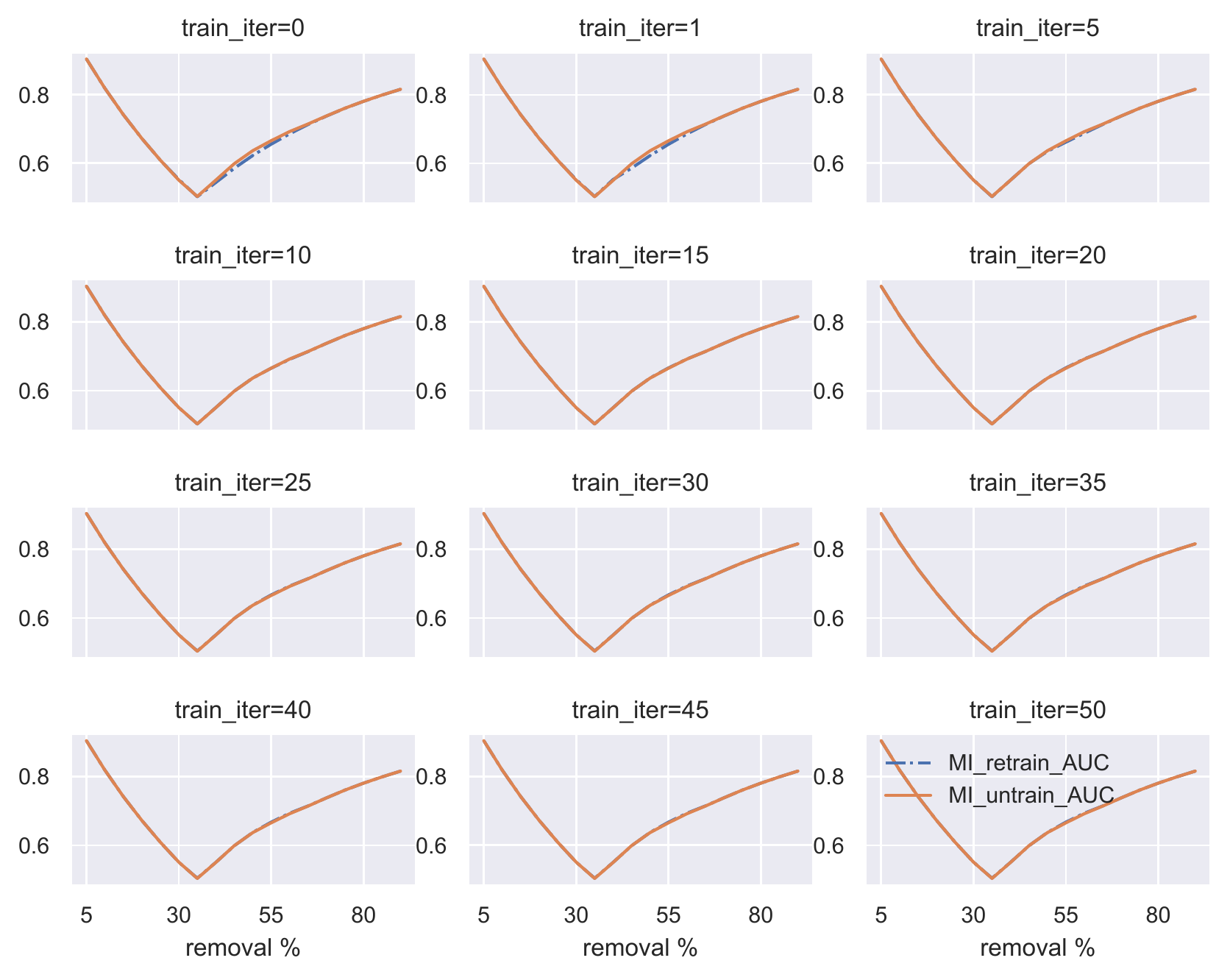}
      \caption{\textbf{Average-case membership inference accuracies.} For 10 passes of Unlearn-ALS, the base models' fit is varied through training and untraining iterations, across 20 different removal fractions in MovieLens-100K.}
      \label{fig:vanilla}
      \end{minipage}
    \end{figure*}

To investigate the practical implication of using Untrain-ALS we examine three aspects:
\begin{enumerate}
\item The accuracy of Untrain-ALS to prevent model degradation. In our case, we show that $\mathcal{M}_\mathrm{untrained}$ performs no worse than $\mathcal{M}_\mathrm{retrained}$ models that result from retraining from scratch on the remaining data.
\item The runtime in practice; in our case, it suffices to show that unlearning takes fewer iterations than retraining.
\item  The privacy implications of Untrain-ALS. In our case, unlearning should reduce privacy risks from undeleted model through reduction on MI accuracies.

Note that empirical privacy evaluations should (i.) reliably uncover vulnerabilities in $\mathcal{M}_\mathrm{undeleted}$, and (ii.) be able to differentiate between $\mathcal{M}_\mathrm{retrain}$, which does not observe the offending data, and $\mathcal{M}_\mathrm{undeleted}$.
\end{enumerate}
For data $P$, we use MovieLens datasets~\citep{bennett2007netflix, harper2015movielens}. On larger models, membership inference suffers severe sensitivity issues, therefore we illustrate on MovieLens-100k and MovieLens-1M. We run parallelized Alternating Least Squares, with conjugate gradient speedup~\citep{hu2008collaborative,takacs2011applications} as baseline; without our inverse adjustments, runtime is apparent through comparing the iteration number between models. Additional setups parameters are outlined in Appendix ~\ref{app:als}.
%The Python implementation is provided at \url{https://anonymous.4open.science/r/aistats_2023-7E20}.

%The removal set, $\mathcal{D}_\mathrm{removal}$, is held out as a fraction of total observations.
%All datasets have a 1:99 test-train split before data removal, so the heldout set is sampled first. Removal fractions ($\%$) are denominated in explicitly observed entries. For training and evaluating recommendation models themselves, we use area-under-curve (AUC), which is more appropriate than downstream recommendations across configurations. For membership attacks, we use vulnerability measures derived from membership inference accuracies~\citep{shokri2017membership}. We show training and untraining across iterations, and evaluate random removal fractions at every 5\% increments.

%To measure the efficacy of Untrain-ALS, we note two baselines: 1. The base model $\mathcal{M}_\mathrm{undeleted}$, which is trained to completion and used to initialize unlearning. It marks ideal accuracy. 2. Re-training from scratch without the removal data, $\mathcal{M}_\mathrm{retrained}$, as the baseline for vulnerability.

\subsection{Untrain-ALS: No Degradation, Fast Convergence}
\label{sec:exp_runtime}
% \begin{figure*}[t]
%   \begin{minipage}[t]{0.47\textwidth}
%     \centering
%     \includegraphics[width=0.98\textwidth]{figure/paper_fig_100k_fit_tr25}
%     \caption{\textbf{Unlearning with Untrain-ALS.} Across different iterations of unlearning, and 20 different fractions of removal, each final model's area-under-curve on test set on MovieLens-100K. Retraining is compared to at 25 passes of ALS.}
%     \label{fig:auc_fit}
%   \end{minipage}
%   \hspace{0.5cm}
%   \begin{minipage}[t]{0.47\textwidth}
%   \centering
%   \includegraphics[width=0.88\textwidth]{figure/paper_fig_100k_overfit}
%   \caption{\textbf{Overfitting Untrain-ALS.} Across different iterations of unlearning, and 20 different fractions of removal, each final model's area-under-curve on test set on MovieLens-100K. Retraining is compared to at 25 passes of ALS.}
%   \label{fig:auc_overfit}
%   \end{minipage}
%   \end{figure*}

    Over a wide range of removal fractions, Figure~\ref{fig:auc_fit} shows fast untraining dynamics, which result in highly performant models. In MovieLens-1M, for removal fractions under 60\%, Untrain-ALS typically converges by $10$ iterations, while retraining takes $40$ to $70$ ALS passes. Because Unlearn-ALS clearly follows the well-tested ALS, if untraining is left unchecked as in Figure~\ref{fig:auc_overfit}, there is no degradation to the model compared to training from scratch.
    Per theoretic analysis in Section~\ref{sec:theory}, Untrain-ALS is consistent with re-training without removal data in objective.
    In summary, Unlearn-ALS breaks the usual expectation that fast unlearning necessarily degrades model performance.

\subsection{Membership Inference (MI) from data queries}
\label{sec:vanilla_mi}

% \begin{figure}[h]
% %\begin{wrapfigure}%[t]
% %\begin{wrapfigure}{r}{0.51\textwidth}
%   \centering
%   \includegraphics[width=0.50\textwidth]{figure/paper_fig_100k_vanilla_un10_square}
%   \caption{\textbf{Raw membership inference accuracies for 10 passes of Unlearn-ALS} The base models' model fit is varied through training and untraining iterations, across 20 different removal fractions in MovieLens-100K.}
%   \label{fig:vanilla}
% %\end{wrapfigure}
% \end{figure}

%Overview
Recall $\mathrm{MI}(\mathcal{M}):(\theta_\mathcal{M}, \mathcal{D}_\mathrm{removal}, \mathcal{D}_\mathrm{remain}) \to [0,1]$. When evaluating data leakage across datasets and model configurations, the measure $\mathrm{MI}(\mathcal{M})$ is \emph{average-case} (Section~\ref{sec:prelim}).
We first motivate MI-based evaluations with a null result, discuss the specific risks of MI for evaluating memorization vulnerabilities, and offer an improved metric. Concurrently, ~\citep{carlini2022membership} states that a more robust metric should be used, though they are motivated by deep learning models.

% MORE BACKGROUND

% DEFINE WHICH VULNERABILITY IS USED.
\paragraph{Average-case MI is insufficient.} In Figure~\ref{fig:vanilla}, after random data is removed, the vulnerability of retraining is comparable to $10$ passes of our proposed unlearning method. While this looks like we are doing perfect, observe even with changes in training iterations (from underfit to convergent), there is no change in vulnerability.  An appealing conclusion is that underfit bi-linear models are always robust to deletion, as though no work needs to be done in unlearning so long as we stop early. \emph{We do not believe that the unlearning algorithm is absolutely private}.

% MORE BACKGROUND ON WHO DOES THIS MI< WHY AM I USING IT
\subsubsection{Sensitivity issue of MI in practice}
\label{sec:mi}
We investigate empirical attacks based on average-case membership inference against the unlearned model. As alluded to in Section~\ref{sec:inherent}, matrix completion-based models have only modest privacy risk to begin with; in implicit feedback datasets, extensive model validation can mitigate the risks against random data deletion. Meanwhile, membership inference attacks~\citep{shokri2017membership} are especially powerful when the input data has a lot of information; in matrix completion, without some advanced user fingerprinting, the model output itself is all the information. Concretely, 3 challenges arise in this pursuit for practical evaluation:
\begin{enumerate}
\item In the real world, the intial pre-trained recommenders tend to perform well, even after large portion of the data is zero-ed (Figure~\ref{fig:auc_1m_retrain}). Uniformly removing data is akin to sampling another held-out set, the initial model likely predicts the missing items just as well (Section~\ref{sec:inherent}). Moreover, corrupting the base model for better privacy is not an option in industrial deployments.

%
% \section{Base Model}
% %  %[width=0.98\textwidth]

% \begin{figure}[t]
%   \centering
%   \includegraphics[width=0.48\textwidth]{figure/retrain_test_auc_1m}
%   \caption{\textbf{Re-training dynamics.} Across different iterations of re-training from scratch, and 10 different fractions of removal, each final model's area-under-curve on test set on MovieLens-1M.}
%   \label{fig:auc_1m_retrain}
% \end{figure}

% Figure~\ref{fig:auc_1m_retrain} shows that even with large removal fractions the base model can still perform well. The dotted vertical lines mark the number of iterations that achieve the best fit for each model. As shown, the less the remaining data, generally the earlier the convergence. 10 passes are sufficient only if removal fraction is large ($>70\%$). For small fractions of removal, the best fit tends to be between $[40,70]$ passes on MovieLens-1M. In comparison, Untrain-ALS only takes $[10,45]$ iterations.

%To make matters worse, ALS performs well even after a large portion of data is deleted.
% as illustrated in Appendix~\ref{fig:auc_1m_retrain}.
%Figure~\ref{fig:auc_1m_retrain} shows that even with large removal fractions the base model can still perform well. As shown, the less the remaining data, generally the earlier the convergence. 10 passes are sufficient only if removal fraction is large ($>70\%$).

\item Using only the predicted value, $\mathcal{M}_\mathrm{untrain}$ does not distinguish between the removed and remaining samples, so there is no significant MI accuracy change. (Only at certain ratios for certain epochs can we induce a 2\% difference.) \emph{MI is highly susceptible to small noise.}
\item Varying train-test splits, the base model $\mathcal{M}_\mathrm{undeleted}$ has different membership attack vulnerabilities built-in, due to ALS not having a fixed unique solution; different training trajectories find different decompositions as solutions to the same matrix completion problem. Some of those models are inherently more defensible than others. This adds noise to the already small numerical measurement. \emph{We tackle this directly.}
\end{enumerate}

Recall that our privacy model views re-training as ground truth. To study the vulnerability of unlearning is to study the \emph{additional} vulnerability compared with re-training.

Let $\mathcal{IV}$ denote the \emph{intrinsic vulnerability} associated with the learning strategy. We are concerned with whether Untrain-ALS presents more or less intrinsic risk compared with retraining. Assuming that the training and un-training procedures have similar intrinsic vulnerability, $\mathcal{IV}_\mathrm{ALS} \approx \mathcal{IV}_\mathrm{re-train}$. An estimator for $\mathcal{IV}_\mathrm{Untrain-ALS}$ is thus the difference between the empirical measure for membership inference:
\begin{equation}
  \label{eq_naive_mi}
  %\begin{split}
  \mathcal{IV}_\mathrm{Untrain-ALS} = \mathrm{MI}(\mathrm{untrain}) - \mathrm{MI}(\mathrm{retrain})
  %\end{split}
\end{equation}
\begin{remark}
Because retraining is assumed to be statistically independent from removed data, being able to infer properties of the removed data from the re-trained model e.g. due to data duplication is not an essential vulnerability. If an empirical measurement shows that untrained model has membership vulnerability, it is a tolerable amount of privacy risk under our setup.
\end{remark}

However, this measurement on intrinsic
untraining vulnerability shows that, at the best fit, untraining and untraining are extremely
close. This numerical difference is so small, that the measurement appears dominated by noise,
while having inconclusive results, as shown in Figure~\ref{fig:vanilla}. When averaged across runs, the overlap of untraining and retraining are further obscured.

\subsubsection{Modifying Average-Case Membership Inference}
\label{sec:modify-mi}
Identifying model noise as a cause, let $\mathcal{IV'}$ be our modified intrinsic vulnerability measure, applied not only to the same $\{M,\mathcal{D}_\mathrm{obs},\mathcal{D}_\mathrm{removal}\}$, but also under identical train-test split. The splits greatly impact the model, as we see in Section~\ref{sec:mi} that intrinsic vulnerability to deletion is closely related to model AUC.
Using ALS and Untrain-ALS to retrain and unlearn after data removal, we make three accuracy measurements: $\mathrm{MI}(\mathrm{untrain})$, $\mathrm{MI}(\mathrm{retrain})$, and $\mathrm{MI}(\mathrm{undeleted})$. Even though our privacy model does not directly concern the base model, the inclusion of $\mathrm{MI}(\mathrm{undeleted})$ serves to \emph{denoise} the influence of model splits on our numerical accuracy differences. We have

% \begin{equation}
%   \begin{split}
%   \label{eq_refined_mi}
% \mathcal{IV'}_\mathrm{Untrain-ALS}
% \\= \mathrm{MI}(\mathrm{untrain}) - \mathrm{MI}(\mathrm{retrain})
% \\ -\mathrm{MI}(\mathrm{undeleted})
%   \end{split}
% \end{equation}

% \begin{equation}
%   \begin{split}
%   \label{eq_refined_mi}
% \mathcal{IV'}_\mathrm{Untrain}
% = \mathrm{MI}(\mathrm{untrain}) - \mathrm{MI}(\mathrm{retrain})
%  -\mathrm{MI}(\mathrm{undeleted})
%   \end{split}
% \end{equation}

\begin{equation}
  \begin{split}
  \label{eq_refined_mi}
\mathcal{IV'}_\mathrm{Untrain} &
=\quad \mathrm{MI}(\mathcal{M}_\mathrm{untrain})
\\ &\quad- \, \mathrm{MI}(\mathcal{M}_\mathrm{retrain})
 -\mathrm{MI}(\mathcal{M}_\mathrm{undeleted}).
  \end{split}
\end{equation}

For the same model, Equation~\ref{eq_refined_mi} appears off by a constant from Equation~\ref{eq_naive_mi}. However, as a measurement, the subtraction for each run improves numerical stability, and reduces noise when averaged over multiple runs.
\begin{figure*}
\begin{minipage}[t]{0.32\textwidth}
  \centering
  \includegraphics[width=0.85\textwidth]{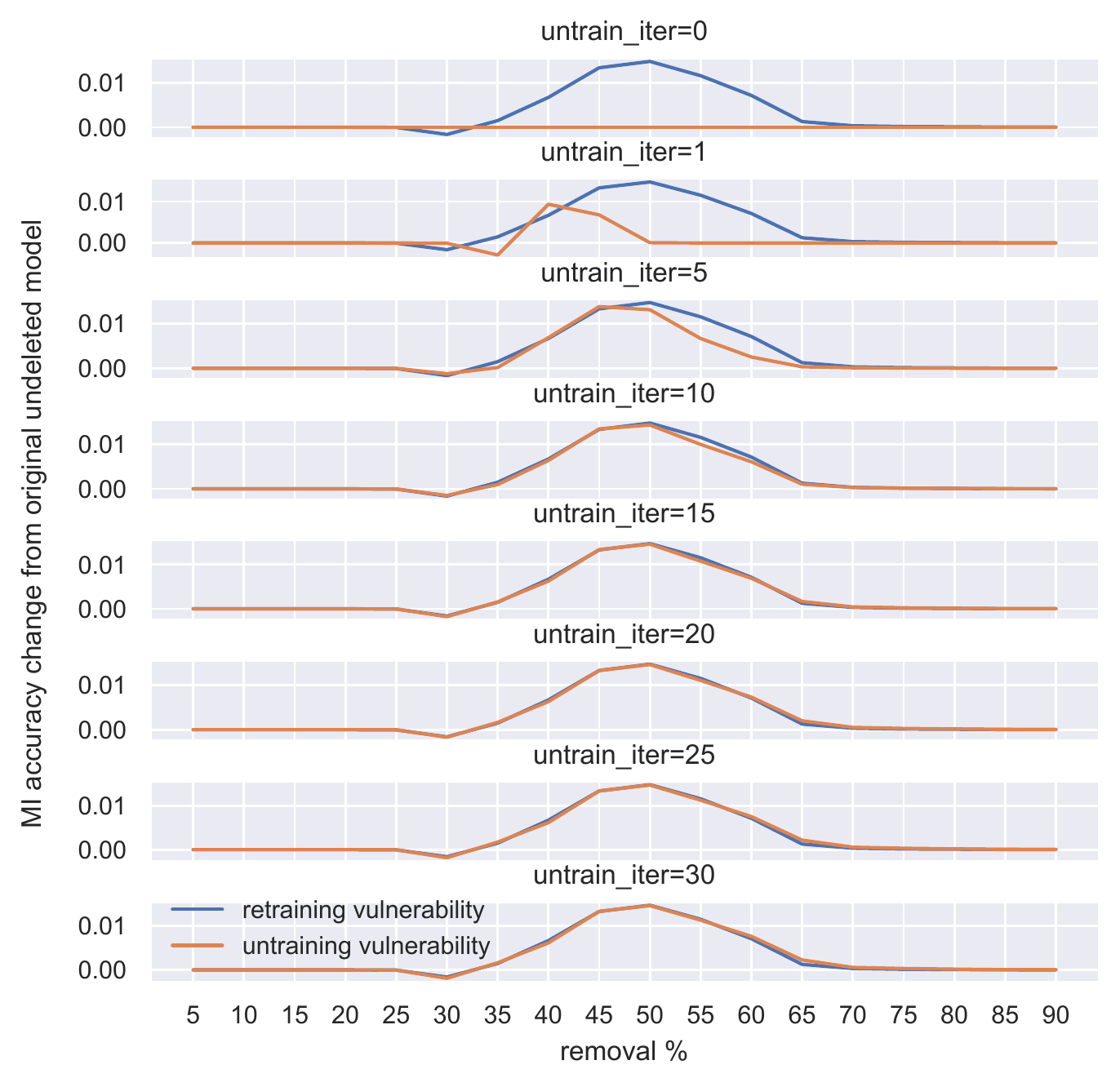}
  % \caption{\textbf{Modified Vulnerability Under Untrain-ALS.} $\mathcal{IV'}$ due to data removal for different untraining iterations and removal fractions, compared against 25 passes of re-training.}
  \caption{\textbf{Untrain-ALS $\mathcal{IV'}$} due to data removal across removal fractions, against $25$ ALS retrain passes.}
  \label{fig:mi_retrain}
\end{minipage}
\hspace{0.05cm}
\begin{minipage}[t]{0.32\textwidth}
 \centering
  \includegraphics[width=0.85\textwidth]{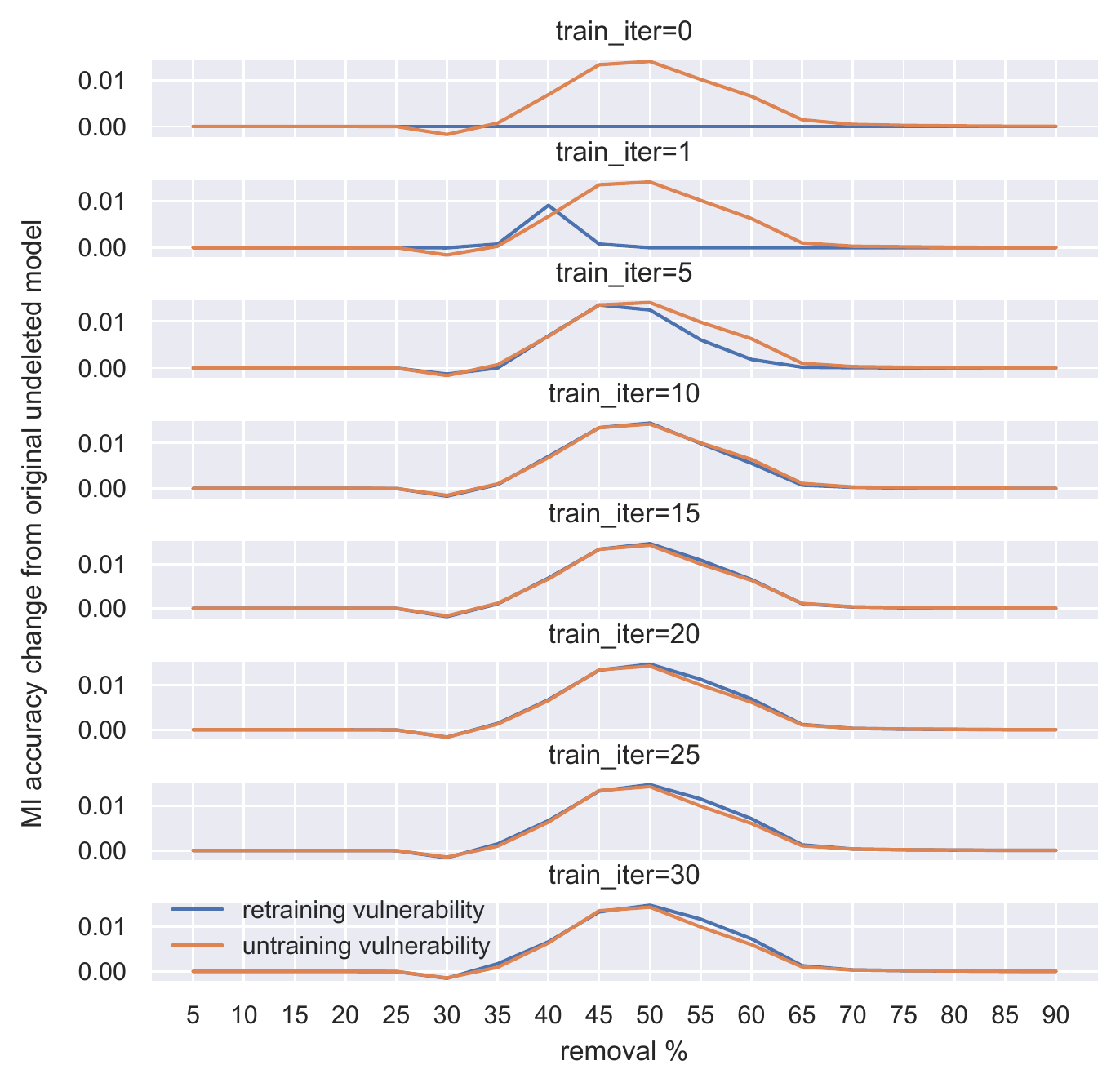}
  \caption{\textbf{Re-training $\mathcal{IV'}$} due to data removal across removal fractions, against $10$ Untrain-ALS passes.}
  \label{fig:mi_untrain}
\end{minipage}
\hspace{0.05cm}
\begin{minipage}[t]{0.32\textwidth}
\centering
\includegraphics[width=0.95\textwidth]{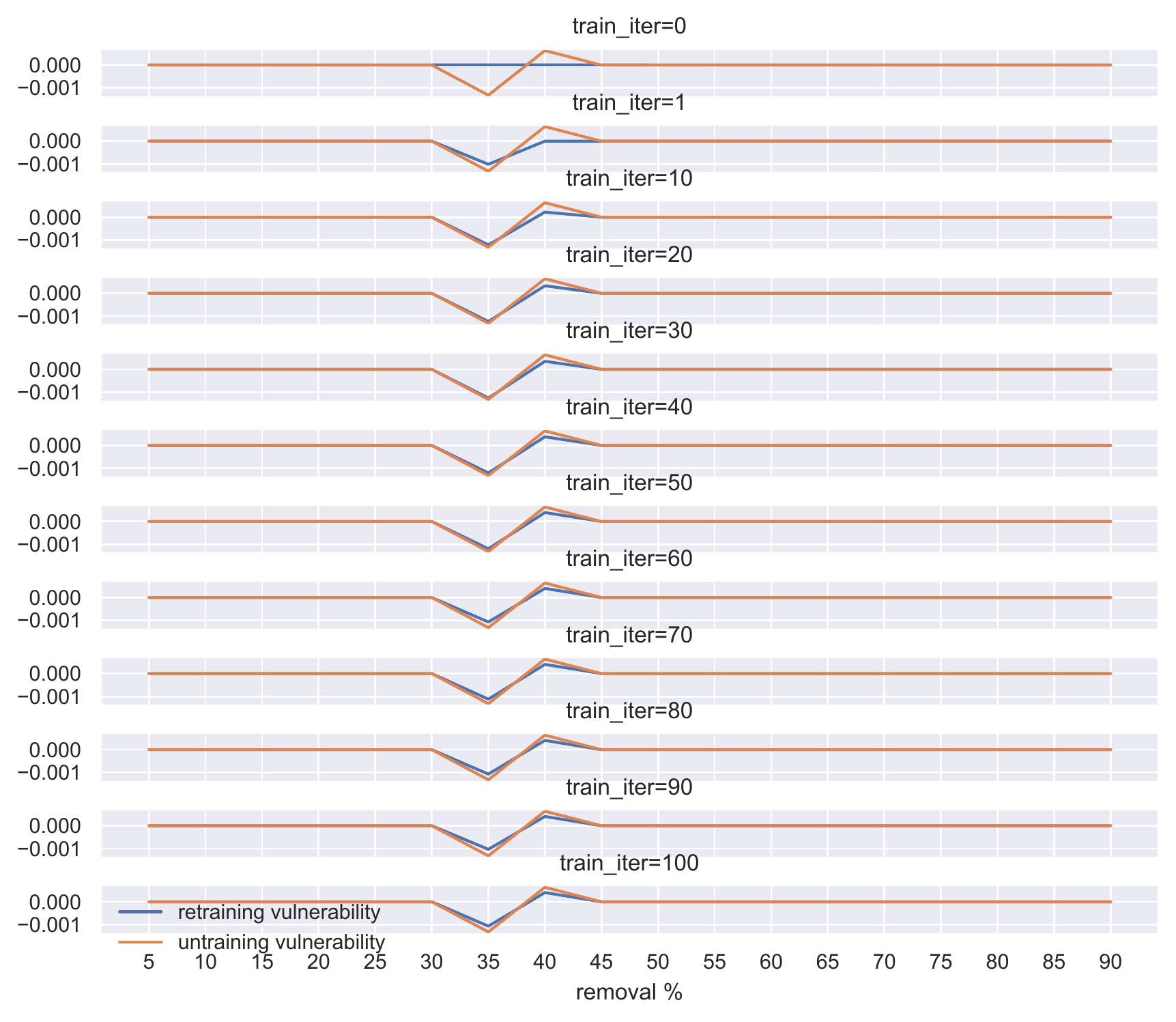}
\caption{\textbf{Large-scale Re-training $\mathcal{IV'}$} on MovieLens-1M. Re-training is compared with $45$ Untrain-ALS passes.}
\label{fig:mi_big}
\end{minipage}
\end{figure*}

In Figure~\ref{fig:mi_retrain} and ~\ref{fig:mi_untrain}, the vulnerability $\mathcal{IV}$ is measured as membership inference accuracy subtracting membership inference accuracy associated with $\mathcal{M}_\mathrm{undeleted}$, for the same split under MovieLens-100K. The removal fraction is set at every 5\% of the data (even though we are empirically only concerned with small fractions). The procedures for untraining involves training the base model with the selected number of iterations.
% The sensitivity of the refined metric is much improved,
\paragraph{Improvements.} $\mathcal{IV'}$ significantly mitigates sensitivity issues, as we now see a change with respect to model training iterations. In Figure~\ref{fig:mi_retrain}, as training iterations get larger, the inherent vulnerability is greater. In Figure~\ref{fig:mi_untrain}, as untraining continues, there is a decrease of vulnerability. Both phenomena are not salient when measured under $\mathcal{IV}$.
\paragraph{Limitations.}
Nonetheless, our efforts to denoise only has a clear effect on small scale on a specific removal range. The range related to user-requested deletion is, however, still not very sensitive. As illustrated in Figure~\ref{fig:mi_big}, at a larger scale, this metric suffers a sensitivity loss.

\section{Discussion, Limitations, and Impacts}
We propose using Untrain-ALS to perform machine unlearning on bi-linear recommendations, which is simutaneously widely-deployed in the real world and under-studied in machine unlearning. The method is fast to converge, and can unlearn exactly without compromising model degradation. However, empirically, models learned with regularized matrix completion are not unique, thus unlearning and re-training may exhibit small differences in privacy. To find them, we employ empirical attacks of memership inference, and adapt the vanilla version to denoise the impact of data splits, and successfully see trends in vulnerability that was previously obscured.
We see two trends emerging from empirical results: 1. Unlearn-ALS is clearly fast and powerful, with no degradation in model performance, unlike most unlearning methods~\citep{sekhari2021remember}. 2. Unlearn-ALS is not the same as re-training, but it closely relates to re-training in most privacy measures, provided that it is trained to the best fit.
2. Relying on membership inference classifications alone to measure unlearning thus leads to potential outstanding privacy risks. We join prior calls in urging the unlearning community to re-think empirical evaluations for unlearning to meet practical privacy needs\citep{truex2019demystifying,chen2021machine, jayaraman2020revisiting}.

\paragraph{Limitations}
\label{sec:limits}
Our work is limited to the choice of models. Though we apply to all bi-linear models, not all recommendation systems are implemented with dot products.

\paragraph{Societal impact}
\label{sec:impact}We place pressure on platforms that train on user data to give users real-time options to remove the influence of their training data. Despite research progress, however, real world systems have yet caught on~\citep{villaronga2018humans}. When users opt to remove their past records' influence on recommendations, existing implementations tend to fall under two categories: complete expunging of their recommendation, in which a user's all historic interactions are zero-ed, such as Netflix's reset, or a vague removal of learnt concepts such as Facebook's Ad preferences. While many services offer granular control over which ones of their historic actions the platform collects, they do not promise that the deletion necessarily impact downstream systems that learn from such data.

Ostensibly, two factors prevent machine unlearning to be deployed: 1. lacking legal recognition for the associated privacy risks, as GDPR-style deletion hinges on whether automated systems leak private data for the general public~\citep{villaronga2018humans}. For that, our work adds to the rigor of discovery: empirical evaluation needs revisiting. 2. industrial-scale computation expenditure on pre-trained machine learning models is massive, and there has yet been a compelling demonstration that industrial-scale recommendation models can be efficiently unlearned without hurting the bottom line. Our work on Untrain-ALS proposes efficient and exact unlearning. Upon sequential deletion requests, the unlearned model will not perform worse than the retrained model. When no trade-off is made, the hope is that both policy and industry can agree to facilitate user privacy.

\paragraph{Malicious use}
Our selective forgetting techniques can be applied to sinister areas where forgetting is a form of censorship. The right to be forgotten also faces similar criticism outside of the legal realms; after all, even the forgetting procedure in "Eternal Sunshine of the Spotless Mind" may be troublesome because it lets the recipient "live a lie"~\citep{grau2006eternal}.

\section{Conclusion}
In practice, matrix completion-based models are not guaranteed to be inherently private. To unlearn, we develop Untrain-ALS, with sound theory and efficient implementations. Untrain-ALS objective aligns exactly with re-training, so developers will see no degradation in model performance caused by choosing unlearning over re-training.

To make our solution practical, we provide a numerical speedup to scaling large systems, and a denoised vulnerability measure to improve membership inference sensitivities.
%Empirically, we show that Untrain-ALS simulates retraining in terms of accuracy measures; however, the typical evaluation metrics, including vanilla membership inference classifications, are poorly suited for evaluating privacy from removal data.
%Lastly, identifying model noise as a cause, we re-define the vulnerability metric to denoise the effect of data splits which can show minute privacy risks. The refined metric is yet only sensitive to a range of deletion, underscoring our call that binary membership inference accuracies are fickle for evaluating unlearning.
%\input{risks}

\bibliography{references}
\bibliographystyle{plainnat}

%%%%%%%%%%%%%%%%%%%%%%%%%%%%%%%%%%%%%%%%%%%%%%%%%%%%%%%%%%%%%%%%%%%%%%%%%%%%%%%
%%%%%%%%%%%%%%%%%%%%%%%%%%%%%%%%%%%%%%%%%%%%%%%%%%%%%%%%%%%%%%%%%%%%%%%%%%%%%%%
% APPENDIX
%%%%%%%%%%%%%%%%%%%%%%%%%%%%%%%%%%%%%%%%%%%%%%%%%%%%%%%%%%%%%%%%%%%%%%%%%%%%%%%
%%%%%%%%%%%%%%%%%%%%%%%%%%%%%%%%%%%%%%%%%%%%%%%%%%%%%%%%%%%%%%%%%%%%%%%%%%%%%%%
\newpage
\appendix
\onecolumn

\section{Proof: Untrain-ALS and Retraining Share The Same Minimal Loss, Functionally.}
\label{app:proof}

Recall that in Alternating Least Squares (ALS), the loss function is the regularized matrix completion:
\begin{equation}
L_\mathrm{ALS}(\mathcal{D}_\mathrm{obs})= \sum_{u, i \in \mathcal{D}_\mathrm{obs}} c_{ui} (p_{ui} - x^\intercal_u y_i)^2 + \lambda(\sum_u||x_u||^2 + \sum_i ||y_i||^2)
\end{equation}
For any set of preference matrix $P$ , deterministic function $f_c$, let the removed dataset be $\mathcal{D}_\mathrm{rm}$. As we only remove explicitly observed data points, it is assumed that $\mathcal{D}_\mathrm{rm}\subset \mathcal{D}_\mathrm{obs}$. When we retrain, we substitute $\mathcal{D}_\mathrm{remain} = \mathcal{D}_\mathrm{obs} - \mathcal{D}_\mathrm{rm}$ for $\mathcal{D}_\mathrm{obs}$, and write the loss under retraining as

\begin{equation}\label{eq:retrain_loss}
L_\mathrm{ALS}(\mathcal{D}_\mathrm{remain})= \sum_{u, i \in \mathcal{D}_\mathrm{remain}} c_{ui} (p_{ui} - x^\intercal_u y_i)^2 + \lambda(\sum_u||x_u||^2 + \sum_i ||y_i||^2)
\end{equation}

When we untrain with Untrain-ALS, we set the confidence values manually to $0$ for the indices in the removal set. We thus have

\begin{equation}\label{eq:untrain_loss}
L_\mathrm{UntrainALS}(\mathcal{D}_\mathrm{obs}, \mathcal{D}_\mathrm{rm}) = \sum_{u, i \in \mathcal{D}_\mathrm{obs}} f_c^\mathrm{untrain}(c_{ui}) (p_{ui} - x^\intercal_u y_i)^2 + \lambda(\sum_u||x_u||^2 + \sum_i ||y_i||^2)
\end{equation}

where $f_c^\mathrm{untrain}(\cdot)$ transforms the confidence score. Using Kronecker delta $\delta$ for set membership, we have
\[f_c^\mathrm{untrain}(c_{ui}) =\delta_{(u,i)\in( \mathcal{D}_\mathrm{obs} \backslash \mathcal{D}_\mathrm{rm})}c_{ui} =  (1-\delta_{(u,i)\in \mathcal{D}_\mathrm{rm}})c_{ui} = c_{ui}-\delta_{(u,i)\in \mathcal{D}_\mathrm{rm}}c_{ui}.\]

Assuming the same removal and observations, we hereby call the two loss quantities on $\{\mathcal{D}_\mathrm{remain}, \mathcal{D}_\mathrm{obs},\mathcal{D}_\mathrm{removal}\}$ in Equation~\ref{eq:retrain_loss} $\mathrm{RETRAIN\_LOSS}$ and in Equation~\ref{eq:untrain_loss}  $\mathrm{UNTRAIN\_LOSS}$. We write $\mathcal{D}_\mathrm{removal}$ and $\mathcal{D}_\mathrm{rm}$ interchangeably.

Our manual zeroing results in

%\begin{align*}
\begin{equation*}
\begin{split}
\mathrm{UNTRAIN\_LOSS}
= \quad &
\lambda(\sum_{u\in \mathcal{D}_\mathrm{obs}}||x_u||^2 + \sum_i ||y_i||^2)
 +
\sum_{(u, i) \in \mathcal{D}_\mathrm{obs}\backslash \mathcal{D}_\mathrm{rm}} f_c^\mathrm{untrain}(c_{ui}) (p_{ui} - x^\intercal_u y_i)^2
\\\\& +
\sum_{u, i \in \mathcal{D}_\mathrm{rm}} f_c^\mathrm{untrain}(c_{ui}) (p_{ui} - x^\intercal_u y_i)^2
\\\\
=\quad&\lambda(\sum_{u\in \mathcal{D}_\mathrm{obs}}||x_u||^2 + \sum_i ||y_i||^2)
 +
\sum_{(u, i) \in \mathcal{D}_\mathrm{remain}} c_{ui} (p_{ui} - x^\intercal_u y_i)^2
\\\\& +
\sum_{u, i \in \mathcal{D}_\mathrm{rm}} (0) (p_{ui} - x^\intercal_u y_i)^2
\\\\
=\quad &\lambda(\sum_{u\in \mathcal{D}_\mathrm{obs}}||x_u||^2 + \sum_i ||y_i||^2)
 +
\sum_{(u, i) \in \mathcal{D}_\mathrm{remain}} c_{ui} (p_{ui} - x^\intercal_u y_i)^2
\\\\
=\quad & \mathrm{RETRAIN\_LOSS}.
\\
\end{split}
\end{equation*}
%\end{align*}

Our objective thus makes our unlearning method \emph{exact} rather than approximate. Recall that the holy grail of unlearning is to approximate retraining. Under these modifications to $p_{ui}$ and $c_{ui}$, we find the loss function of Untrain-ALS is functionally equivalent to re-training, derived in Appendix~\ref{app:proof}
. The extraneous terms relating to removal data is fully zero-ed. We thus claim that optimizing Untrain-ALS \emph{can} achieve the same loss as retraining without the removal data.

\begin{remark}
Whether that minimal loss is achieved, and whether the solutions at minimal loss are necessarily equivalent (or up to a rotation) are not guaranteed from this analysis.
\end{remark}

\section{Experimental Parameters}
\paragraph{Alternating least squares.}
\label{app:als}
$P$ is a preference matrix, which could be binarized values of 1 (like) or 0 (dislike). A confidence score $c_{ui}=f_c(p_{ui})$, where $f_c$ is deterministic. In our experiments, $c_{ui} = 1$ if $p_{ui} = 1$, and 0 or very small otherwise. In the paper it is assumed that $f_c(p_{ui}) = 1 + \alpha p_{ui}$ with a suggested $\alpha=40$. Each experiment starts with new seed, including train-test split and ALS initializations, unless otherwise mentioned. Graphs are made with 5 runs.

The removal set $\mathcal{D}_\mathrm{removal}$ is assumed to be uniformly sampled from $\mathcal{D}_\mathrm{obs}$ without replacement.

The number of ALS passes ("epoch" or "iter") is the only tunable parameter for fitting base models. We assume a 99-1 split of train-test, and select epochs based on the best fit validated AUC.

\paragraph{Untrain-ALS Experiments}
We use MovieLens datasets~\citep{bennett2007netflix, harper2015movielens} as our preference matrices. For membership inference, the sensitivity issue is severe on larger models, therefore we illustrate with smaller datasets. Unless otherwise specified, we use parallel implementations for Alternating Least Squares, with conjugate gradient speedup~\citep{hu2008collaborative,takacs2011applications}; without our inverse adjustments, we only compare the number of iterations between untraining and retraining.
%The Python implementation is provided at \url{https://anonymous.4open.science/r/aistats_2023-7E20}.

%The removal set, $\mathcal{D}_\mathrm{removal}$, is held out as a fraction of total observations.
All datasets have a 1:99 test-train split before data removal, so the heldout set is sampled first. Removal fractions ($\%$) are denominated in explicitly observed entries. For training and evaluating recommendation models themselves, we use area-under-curve (AUC), which is more appropriate than downstream recommendations across configurations. For membership attacks, we use vulnerability measures derived from membership inference accuracies~\citep{shokri2017membership}. We show training and untraining across iterations, and evaluate random removal fractions at every 5\% increments.

To measure the efficacy of Untrain-ALS, we note two baselines: 1. The base model $\mathcal{M}_\mathrm{undeleted}$, which is trained to completion and used to initialize unlearning. It marks ideal accuracy. 2. Re-training from scratch without the removal data, $\mathcal{M}_\mathrm{retrained}$, as the baseline for vulnerability.
\paragraph{Membership inference.} The numbers of iterations for the base models are chosen for validated best model fit, as to be expected for practical deployments.

We use the 50-50 split for test-train on removal and remaining datasets for each appropriate removal fraction, meaning that $50\%$ of the removal data is used in training while the rest is used to validate. The best AUC is taken on the removal data for reporting each model's membership attack accuracy.

\section{Base Model}

\begin{figure}[t]
  \centering
  \includegraphics[width=0.78\textwidth]{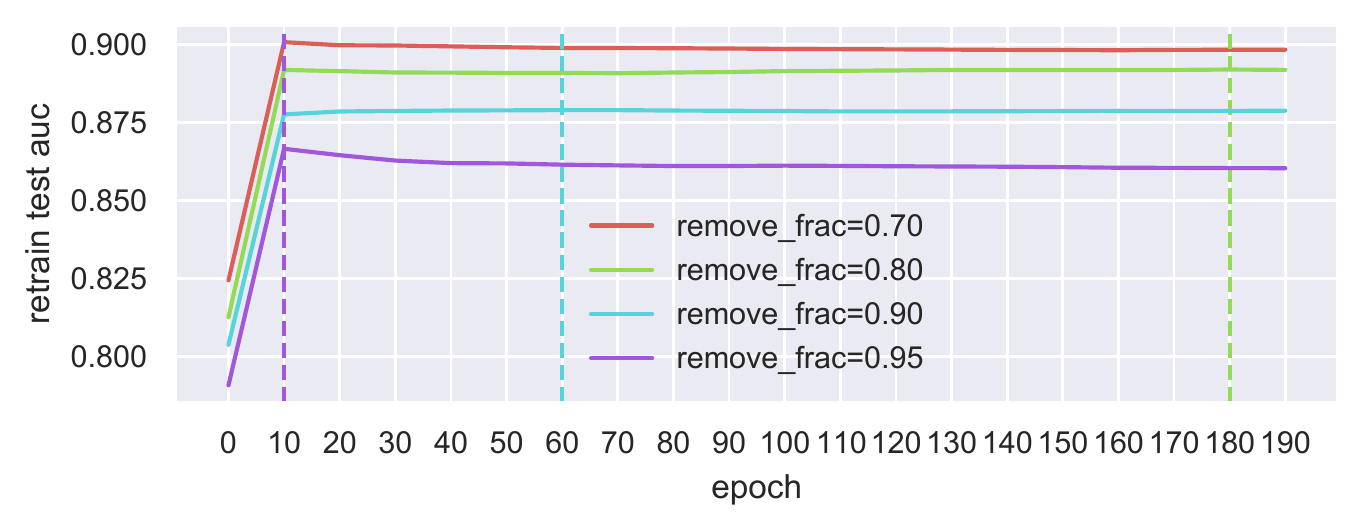}
  \caption{\textbf{Re-training dynamics.} Across different iterations of re-training from scratch, and 4 different fractions of removal, each final model's area-under-curve on test set on MovieLens-1M.}
  \label{fig:auc_1m_retrain_app}
\end{figure}

Figure~\ref{fig:auc_1m_retrain} and Figure~\ref{fig:auc_1m_retrain_app} shows that even with large removal fractions the base model can still perform well. The dotted vertical lines mark the number of iterations that achieve the best fit for each model. As shown, the less the remaining data, generally the earlier the convergence. 10 passes are sufficient only if removal fraction is large ($>70\%$). For small fractions of removal, the best fit tends to be between $[40,70]$ passes on MovieLens-1M. In comparison, Untrain-ALS only takes $[10,45]$ iterations.

% chart

\section{Inherent Privacy (Extended)}
\label{app:inherent}
Is there a scenario where the recommendation model is "robust" to random user deletion, thus requiring no additional work to unlearn a small subset of training data? Intuitively, dimensionality reduction should result in models with low capacity to memorize. Arguably, high empirical performance also relates to inductive biases: if datasets are well-described by the learned low rank parameters (that exhibit good generalization), it should imply that the model's inductive bias is not to memorize. We make concrete these intuitions in the context of matrix completion for user-movie preferences.
%Lacking proper analysis we may see that most of the data is private, but the small set of requested samples to delete is not.

For the following proof sketches, we assume that data removals are independently selected i.e $\mathcal{D}_\mathrm{removal}$ to be sampled from $\mathcal{D}_\mathrm{obs}$ randomly without replacement.

\paragraph{Validation implies robustness to missing data.}
First, a key observation: in implicit feedback datasets, each unobserved (and deleted) user-movie interaction is changed to $0$. The empirical validation of the model relies on a train-test split that follows the same zeroing convention.

As the mechanism for selecting missing feedback is equivalent to selecting a \emph{held-out} set, any argument for in-domain generalization from appropriate calibration by~\citet{kearns1995bound} and ~\citet{blum1999beating} would imply low prediction losses on missing data for both retrained and undeleted models.

Recall that membership inference needs to succeed by discriminating the predictions from removal data and the remaining data. Varying data splits, a well-calibrated model has similar expected losses. Because optimizing area-under-curve (AUC) is used for both 1. thresholding membership inference model on the removal data and 2. on remaining validation data on the base retrained recommendation model, we have $P_\mathrm{retrained}(p_{ui} = 1| (u,i)\sim \mathcal{D}_\mathrm{rm})\approx P_\mathrm{retrained}(p_{ui} = 1| (u,i)\sim \mathcal{D}_\mathrm{obs}) = \mathrm{AUC}_\mathrm{retrain}$. For each model, the approximation is directly relatable to validation loss.

If the base model is highly accurate i.e. has high AUC, the nonnegative loss contribution from removal data is further limited. Empirically, most recommendation data achieves high AUC even with large fractions of data removed. As membership inference needs to discriminate two sets of small, non-negative numerical losses of similar means, the task is inherently hard.

Roughly speaking, a well-validated model indeed implies robustness to deletion of a small fraction of data. Implicit feedback models are unique, where cross-validated performance implies an upperbound on the expected removal data's loss contribution, provided that the deletions are independent.

\begin{remark} Though this property makes empirical evaluation for individual privacy harder, it does also mean that the work towards validation and calibration applies directly towards model robustness against deletion. Even though model noise is inevitable in real-world setting, this insight greatly reduces the expectation that there is unknown privacy risk that result from deletion, as all training data is already observed (and presumably pre-selectable for validation).
\end{remark}

\paragraph{Uniqueness of matrix completion implies robustness to missing data.}
In light of that, we rehash ~\citet{recht2011simpler}'s work to recommendation data. We hereby ignore the coherence requirement on the raw data, which is likely untestable in practice (despite fast approximations~\citep{drineas2012fast, mohri2011can}); instead, assume the row and column spaces have coherences bounded above by some positive $\mu_0$, as assumptions in Theorem 2 in ~\cite{recht2011simpler}. The setup readily applies to our problem. As we assuming that the observations are indeed low-rank, the recovery of the true matrix is certainly robust to small fractions of random deletions.

For preference matrix $P$ of dimension $m\times n$ where $m<n$, assume that underlying ground truth matrix $M$ records the true preferences. Because the preferences are low rank, there is rank $r$ and a singular value decomposition $M= U\Sigma V^*$. As any preference entry is bounded, we trivially obtain the constant value $\mu_1 := \nicefrac{mn}{r}$; in practice $\mu_1 \geq 1e5$ or greater for very sparse datasets.

Assuming a threshold probability is chosen, so that the resulting matrix completion $UV^*$ to the problem

\begin{equation}
\begin{aligned}
\,\min_{U,V} \quad & ||X||_\mathrm{nuclear}& \quad\\
\textrm{s.t.} \quad & P_{ui} = M_{ui} & (u, i)\in \mathcal{D}_\mathrm{obs}\\\\
\end{aligned}
\end{equation}
is unique and equivalent to $M$ with the given probability when the number of uniformly random observations reaches $\mu_2 m n$ for $\mu_2 \leq 1$. This re-hashes the original result without the explict writeout of the bounds on $\mu_2$ that depends on $\{\mu_0, \mu_1, r, m, n\}$ and the chosen probability threshold.

Let $|\cdot|$ denote set cardinality and let $q$ be the fraction of missing data upon user-requested deletions, so that $|\mathcal{D}_\mathrm{removal}| = q |\mathcal{D}_\mathrm{obs}|$. Given $\mu_2$, missing data simply subtracts from the number of total observations. When the size of the remaining data, $(1-q) |\mathcal{D}_\mathrm{obs}|$, is above $\mu_2 mn$, the recovery is yet unique. That means our missing data does not change the uniqueness condition, if $q\leq 1-\frac{\mu_2}{ |\mathcal{D}_\mathrm{obs}|}mn$.

Finally, we want show that for sufficient number of observations, matrix completion solutions are inherently robust. Consdier the retrained and undeleted models. Our results show that they may have the same decomposition under our assumptions, meaning that retraining would not alter the resulting recommendation system in terms of recovered entries i.e. predictions downstream. Their empirical privacy is thus equivalent, meaning the undeleted model is as private as the retrained model.

\begin{remark}
Unfortunately, the bound is often vacuous, as the real world data is far sparser than what the theoretics posit i.e. $\mu_2\propto \mu_1^2$ while $\mu_1$ is too large. Additionally, the minization is often performed using heuristic methods such as alternating least squares, where the uniqueness of the solutions is not guaranteed, even if the underlying un-regularized minimization is unique.

For practical privacy, the independece assumptions of random independent romoval can not be guaranteed; after all, many users will likely remove the most embarassing content from watch history.
\end{remark}

\section{Membership Inference Metric}
For our application context, the natural measure is whether an observer of model outputs can recover or guess what a user once sought to remove.

Divergence-based measures aim to see the downstream difference between untrained and retrained models using a divergence measure $D(P_\mathrm{retrain}||P_\mathrm{retrain})$, such as KL-divergence~\citep{golatkar2020eternal}. At evaluation it is hoped that $\forall (u,i)\in\mathcal{D}_\mathrm{removal}$,
\begin{equation}
\label{eq:naive_div}
P_\mathrm{retrain}(p_{ui} = 0) = P_\mathrm{retrain}(p_{ui} = 0).
\end{equation}

However in collaborative filtering, this objective is under-constrained, as the adversary can observe outputs outside of those in $\mathcal{D}_\mathrm{removal}$ which may be impacted through the removal process. Even if those removed data points remain similar in output, an adversary may still see from the remaining data some anomalies. Instead, suppose an eavesdropper who can observe all data that is observed, except for a particular entry $p_{u_0 i_0}$, we have $\forall (u,i)\in\mathcal{D}_\mathrm{obs}$,
\begin{equation}\label{eq:mi_div}
P_\mathrm{retrain}((u,i)\thicksim \mathcal{D}_\mathrm{removal}) = P_\mathrm{retrain}((u,i)\thicksim \mathcal{D}_\mathrm{removal}).
\end{equation}

Thus we use membership attack to empirically calibrate both sides, maximizing the probability of attack success for a given model, and then measure the difference between those optimal success rates. For an appropriately forgotten model i.e. complete and not-deleting, the membership attack rate should not increase for the "best guess" for any data removed from the preference matrix.

Two benefits ensue: 1. auto-calibration that is suitable for our threat model, when Equation~\ref{eq:naive_div} is uncalibrated, and 2. usability when we only have two models per data split, instead of relying on sampling from a distribution of models.
\section{Privacy Analysis (Extended Discussion)}
\subsection{Privacy Context, Threat Model, and the Legality of Data Removal}
\label{app:privacy}
User privacy is a complex issue deserving of nuanced debate. We hereby outline related concepts.
\paragraph{Privacy in "Netflix and Forget".}
The data flow in our privacy model originates from the user, while the adversary also includes the user. It deviates from common privacy notions such as preventing information extraction~\citep{diffie1979privacy}, or the Right To Be Forgotten~\citep{rosen2011right}.

However, our privacy motivation is a pragmatic user scenario. While being private from one's own recommendations is not considered "unauthorized", letting other users guess the original data with high likelihood constitutes as unauthorized after the data source is withdrawn.

Even though the legal implements of the right to be forgotten are limited, forgetting past records at user request is a natural form of privacy. While most cases discussed under the rightinvolve public records, ~\citet{powles2015google} argues that the system through which the information is surfaced is crucial. Though people may prefer personal data removed purely out of emotional reasons, computational systems often treat data with "decontextualized freshness":

\begin{quote}
They include prominent reminders that an individual was the victim of rape, assault or other criminal acts; that they were once an incidental witness to tragedy; that those close to them – a partner or child – were murdered. The original sources are often many years or decades old. They are static, unrepresentative reminders of lives past, lacking the dynamic of reality.~\cite{powles2015google}
\end{quote}

We thus take the right to be forgotten in the spirit of decaying information while giving users the autonomy over their data. When the data is forgotten, we expect the system to behave as though the data was not supplied in the first place. On the other hand, to devise an attack, we use membership attack under the model that an observer of the recommendation system should not be able to tell with high probability whether some information was \emph{removed}.

\textbf{Threat Model}
The data owners request random deletion of training data, to which the model owner respond by updating the model. An eavesdropper with access to the model outputs attempts to guess whether a data point had been removed.

\paragraph{Does machine learning need to implement the Right To Be Forgotten?}
The ability to remove personal digital records is grounded in normative ethics. In dealing with loss of loved ones, common bereavement guides suggest removing the audio retained from answering machines, as voices, unlike photos, are often recorded incidentally rather than for the sake of remembranace~\citep{massimi2010death}. To move on from grief, a human user ought to have the ability to remove past records that bring them horror and regret, including the records' downstream summaries or re-caps.

However, the current paradigms of the legal system in the United States, where many major techonology companies are based, do not support a comprehensive regulation on privacy specific to machine learning systems.%~\citep{KERRY}
The California Consumer Privacy Act (CCPA) ~\citep{california2018california} and the proposed congressional bill Consumer Online Privacy Rights Act (S.3195) ~\citep{congress_2021} forbid businesses from expanding processing of personal data beyond the intended use. They are, however, quite fresh and rarely enacted. Meanwhile, the more mature GDPR supports such removal of past records used in "automated decisions"~\citep{EUdataregulations2018}. Nevertheless, ~\citet{burt_2018} interprets that even though users usually need to consent to their data being used for training, removing it does not necessarily mean the models need to be retrained.

A case may hinge on whether the un-removed model will leak information about the data to remove~\citep{burt_2018}. While \citet{papernot2016semi, choquette2021label} have shown that many models being deployed such as large scale language models have concrete privacy risks, such tools of empirical evaluation is not accessible to the general public, especially when they rely on accessing the training process. At best they serve as self-checking tools for companies that decide to provide such feature, but not as a tool that can be incorporated into regulation. The current state of online privacy is thus in a state of disarray: a lot of private data is compromised, which are fed to machine learning models. Meanwhile, few regulations are put in place to deal with the downstream effect, and no publically accessible method to measure such privacy loss.

\subsection{Why Aren't Machine Unlearning Solutions Deployed In Machine Learning?}
As ~\citet{waldman2020cognitive} observes, deploying privacy features that match the user's cognitive model is not a priority for technology developers. While many users would likely remove historical records on YouTube or Netflix hoping for changed recommendations, few recommendation systems have transparent guarantees on unlearning user preferences.

\paragraph{Legal recognition} is the most ostensible obstacle: only a few privacy bills have been passed in America, where many major technology companies are located. Lacking any aforementioned privacy regulation specifically worded on artificial intelligence, there is little recourse for users who want their data removed from machine learning pipelines.

\paragraph{Industrial-scale computation} is one reason lobbists use against passing bills that compel real-time removal of user data. Retraining is considered expensive, thus bad for business. While it may be argued that user privacy holds priority over computation cost and model accuracy, there has yet been a compelling demonstration that industrial-scale recommendation models can be efficiently unlearned without hurting the bottom line. After all, large recommendation models are widely used in multiple downstream products, and are expensive to train and re-train.

\paragraph{Flexible unlearning.} Undoubtly, the holy grail of machine unlearning is to allow any model to forget arbitrary training data, as if it were re-trained from scratch. Such a method, which does not depend on a specific learning architecture, would have truly sweeping implications. Generic unlearning applies to a wide range of models, without incurring costly training time modifications, extensive check-pointing, or differential private training. Moreover, with the popularity of finetuning pretrained models for applications, the downstream model servers may not have access to the training procedure or original parameters to begin with. Unlearning without learning enables most trendy services to fortify their systems after performing finetuning.

Additionally, our work uncovers a different dimension of the issue: evaluations. We need a way to know when privacy is lost, and when privacy is perserved.

\subsection{Auditable Evaluations}
A mature unlearning system would need to have compelling and robust evaluations. We still lack a realistic and auditable alternative to membershio inference~\citep{thudi2021necessity}. When un-learning simulates re-training, the ground truth to compare against is clear and reasonable. Platforms and regulators would only need to communicate the following idea: data deletion from machine learning model is analougous to forgetting, acting as if the platform never received such data.

Against privacy risks, a defended model needs to be evaluated against the identified risk. Membership Inference aims to identify memorization of training data by a model, and has gained popularity in succeeding in uncovering privacy risks~\citep{shokri2017membership, rahman2018membership, truex2019demystifying, choquette2021label}. Typical membership inference uses a collection of samples that are not in the training data, feed them to the model, and take the outputs as the baseline negative training set. The positive training set is the data that the model has seen in the training set. Other membership inference methods have been developed, usually requiring access to the model or the training procedure or a more focused clean dataset~\citep{long2020pragmatic, rahimian2020sampling, ye2021enhanced}. The central idea is to make the empirical attack prediction more salient more powerful adversaries.

Recently, ~\citet{46702} took a different approach for large scale language models to test if a data point had been deleted~\citep{46702, izzo2021approximate}. This negative dataset is manufactured "poison" to the training procedure. The intuition is that if the model is prone to memorization, it would be able to reproduce the exact random string that was injected in the training set. The membership inference variant thus focuses on engineering a better dataset, thus making it more effective at uncovering memorization. While powerful in engineering a clear metric, this approach requires the model owner to audit from within.

Our scenario for recommendation privacy turns out especially revealing: common membership inference classification is not able to uncover privacy risk, even though the devised method is not information-theoretically private. Indeed ~\citet{jayaraman2020revisiting} calls for revisiting membership inference in real life, noting that it is not as powerful as an empirical measure. ~\citet{chen2021machine} points out that unlearning can, in fact, decrease privacy, highlighting the need for better evaluations. We thus join calls with ~\cite{thudi2021necessity} in calling for auditable algorithms that evaluate machine unlearning.

\bibliographystyle{icml2023}

\end{document}